\documentclass{article}

\usepackage{graphicx,xcolor,wrapfig,subcaption}
\usepackage{amsfonts,amsmath,amsthm,amssymb}
\usepackage{bm}
\usepackage[capitalize]{cleveref}
\usepackage{booktabs,threeparttable,makecell}
\usepackage{multicol,multirow}
\usepackage{thmtools,thm-restate}
\usepackage{etoolbox}

\newif\ifinappendix
\inappendixfalse

\pretocmd{\appendix}{\inappendixtrue}{}{}

\theoremstyle{plain}
\newtheorem{theorem}{Theorem}[section]

\newtheorem{lemma}[theorem]{Lemma}

\theoremstyle{definition}

\theoremstyle{remark}

\newcommand{\dataset}[1]{\mathcal{#1}}

\newcommand{\result}[2]{#1{\tiny$\pm$#2}}
\newcommand{\kldiv}[1]{\operatorname{KL}\left(#1\right)}
\newcommand{\gjsdiv}[2]{\operatorname{JS}_{#1}\left(#2\right)}
\newcommand{\tvdist}[1]{\operatorname{TV}\left(#1\right)}
\newcommand{\domain}[1]{\operatorname{dom}\left(#1\right)}

\newcommand{\E}{\mathbb{E}}

\newcommand{\errF}{\epsilon^{\operatorname{F}}}
\newcommand{\errCF}{\epsilon^{\operatorname{CF}}}

\newcommand{\method}{CAETC}
\newcommand{\methodlower}{causal autoencoding and treatment conditioning}
\newcommand{\methodcapitalize}{Causal Autoencoding and Treatment Conditioning}
\newcommand{\methodcell}[1]{\makecell{ CAETC \vspace{-2pt} \\ {\tiny #1}}}

\usepackage[preprint]{icml2026}

\icmltitlerunning{\methodcapitalize}

\begin{document}

\twocolumn[
  \icmltitle{\method: \methodcapitalize \\ for Counterfactual Estimation over Time}

  \begin{icmlauthorlist}
    \icmlauthor{Nghia D. Nguyen}{siebel,vishc}
    \icmlauthor{Pablo Robles-Granda}{siebel}
    \icmlauthor{Lav R. Varshney}{sb,ece}
  \end{icmlauthorlist}

  \icmlaffiliation{vishc}{VinUni-Illinois Smart Health Center, VinUniversity}  
  \icmlaffiliation{siebel}{Siebel School of Computing and Data Science, University of Illinois Urbana-Champaign}
  \icmlaffiliation{ece}{Department of Electrical and Computer Engineering, University of Illinois Urbana-Champaign}
  \icmlaffiliation{sb}{AI Innovation Institute, Stony Brook University}

  \icmlcorrespondingauthor{Nghia Nguyen}{nghiadn2@illinois.edu}

  \vskip 0.3in
]

\printAffiliationsAndNotice{}  

\begin{abstract}
Counterfactual estimation over time is important in various applications, such as personalized medicine. However, time-dependent confounding bias in observational data still poses a significant challenge in achieving accurate and efficient estimation. We introduce {\methodlower} ({\method}), a novel method for this problem. Built on adversarial representation learning, our method leverages an autoencoding architecture to learn a partially invertible and treatment-invariant representation, where the outcome prediction task is cast as applying a treatment-specific conditioning on the representation. Our design is independent of the underlying sequence model and can be applied to existing architectures such as long short-term memories (LSTMs) or temporal convolution networks (TCNs). We conduct extensive experiments on synthetic, semi-synthetic, and real-world data to demonstrate that {\method} yields significant improvement in counterfactual estimation over existing methods.

\end{abstract}

\section{Introduction}

Personalized medicine requires knowledge about individualized responses to potential treatments for effective treatment planning \cite{kentPersonalizedEvidenceBased2018,feuerriegelCausalMachineLearning2024}. While randomized controlled trials (RCTs) are the gold standard for treatment effect estimation, RCTs in practice are often costly and difficult to implement. Digitization enables more healthcare data to be recorded, especially via electronic health records (EHRs). Combining large data volumes with appropriate modeling can enable accurate estimation of counterfactual outcomes and treatment effects, forming the basis for individualized decision making in health care \cite{bareinboimCausalInferenceDatafusion2016} and other domains.

Modern counterfactual estimation methods integrate frameworks from causal inference and deep learning for counterfactual estimation over time \citep{limForecastingTreatmentResponses2018,bicaEstimatingCounterfactualTreatment2019,melnychukCausalTransformerEstimating2022,bouchattaouiCausalContrastiveLearning2024,wangEffectiveEfficientTimeVarying2024}. However, counterfactual estimation over time is difficult due to time-dependent confounding bias, where covariates that affect treatment choices evolve dynamically and are themselves influenced by prior treatments \citep{robinsMarginalStructuralModels2000}. This results in systematic differences in the distribution of confounders between treatment regimes, violating key identification assumptions and complicating the estimation of unbiased causal effects. Conventional time-series models lack the capacity to adjust for this bias, relying on advanced methods to disentangle confounders from treatment assignment. 

Learning to balance or disentangle time-dependent confounders is crucial for unbiased counterfactual estimation over time. Modern methods leverage sequence architectures like recurrent neural networks, transformers, or state-space models, with de-confounding mechanisms, like inverse probability of treatment weighting or adversarial balancing, to mitigate time-varying confounding bias, enabling better long-term treatment effect prediction in complex, sequential decision-making settings \citep{limForecastingTreatmentResponses2018,bicaEstimatingCounterfactualTreatment2019,melnychukCausalTransformerEstimating2022,bouchattaouiCausalContrastiveLearning2024,wangEffectiveEfficientTimeVarying2024}. For instance, recurrent marginal structure network (RMSN) \cite{limForecastingTreatmentResponses2018} uses an LSTM architecture with inverse probability of treatment weighting (IPTW). Alternatively, counterfactual recurrent network (CRN) \cite{bicaEstimatingCounterfactualTreatment2019} uses LSTM architecture with a gradient reversal layer, whereas causal transformer (CT) \cite{melnychukCausalTransformerEstimating2022} has a transformer architecture with domain confusion loss to learn explicit treatment-invariant representation.

Despite these remarkable advances, several key limitations persist. First, existing methods (CRN, CT) suffer from the loss of covariate information due to adversarial training \citep{huangEmpiricalExaminationBalancing2024}. Achieving treatment invariance with an expressive, invertible representation that encodes sufficient information on history for identification of causal effects remains a challenge. Consequently, models risk learning shortcuts that limit their ability to recover individual-level causal responses or suffer from loss of heterogeneity \cite{melnychukBoundsRepresentationInducedConfounding2023}. More recent methods, such as causal contrastive predictive coding (CCPC) \cite{bouchattaouiCausalContrastiveLearning2024}, tackle the problem implicitly by using the InfoMax principle to encourage representation invertibility. Conversely, Mamba-CDSP \cite{wangEffectiveEfficientTimeVarying2024} uses an architecture-specific loss to decorrelate history with planned treatment to circumvent adversarial training. Second, existing methods put little emphasis on modeling the interaction between the planned treatment and the balanced representation to decode future outcomes. While the representation balancing stage yields a theoretically valid representation, the interaction between the representation and planned treatment directly affects the future outcome prediction (see \cref{fig:architecture}A). By making this interaction explicit, we intuitively and simultaneously encourage representation invertibility and predict future outcomes without excessive architecture changes or complex training procedures. 

Inspired by these recent methods and their constraints, we present an architecture-agnostic method called \emph{{\methodlower} (\method)} for counterfactual estimation over time. Our methodological contributions are the following:
\begin{itemize}
    \item We design a model-agnostic method for counterfactual estimation over time that encourages a partially invertible representation via autoencoding and predicts future outcomes via treatment conditioning. We further encourage the conditioning layer to learn a treatment-specific transformation to improve counterfactual estimation performance. We apply the proposed design on long-short term memory (LSTM) and temporal convolution network (TCN). 
    \item We propose an entropy maximization adversarial game that yields theoretically balanced representation across all treatment regimes. The adversarial game is equivalent to minimizing a generalized Jensen-Shannon divergence between treatment-conditional representation distributions. We further demonstrate that under certain assumptions and distributional conditions, the outcome estimation error is bounded by the aforementioned Jensen-Shannon divergence term.
    \item We empirically validate {\method} on synthetic, semi-synthetic, and real-world datasets to demonstrate that our method achieves strong improvement over existing counterfactual estimation baselines.
\end{itemize}

\section{Related Works}

Counterfactual estimation has received great attention in both static and temporal settings. 

Here, we describe prior methods in the temporal domain.

\subsection{Counterfactual estimation over time} 

Methods that control for time-dependent confounding were originally developed in epidemiology for longitudinal studies, such as the structural nested model (SNM) \cite{robinsCorrectingNoncomplianceRandomized1994} and marginal structural model (MSM) \cite{robinsMarginalStructuralModels2000}. However, these methods have traditionally been implemented using simple parametric estimators with limited capacity to deal with complex, high-dimensional data. RMSN improves over MSM by incorporating LSTM to handle nonlinearities in the data. However, RMSN still relies on IPTW and suffers from the same problems, such as high variance. Alternatively, G-Net \cite{liGNetRecurrentNetwork2021} adapts g-computation \cite{robinsNewApproachCausal1986, robinsGraphicalApproachIdentification1987} to sequence model. A class of recent methods, such as CRN and CT \citep{bicaEstimatingCounterfactualTreatment2019, melnychukCausalTransformerEstimating2022}, uses domain adversarial training to learn a treatment-invariant representation, removing the systematic differences between treatment regimes. Further progress has been made toward different approaches to handle time-dependent confounding, such as adapting ODE discovery to treatment effect estimation \cite{kacprzykODEDiscoveryLongitudinal2023}. 

\subsection{Learning treatment-invariant representation}

Recent approaches adopt adversarial treatment-invariant representation learning \citep{bicaEstimatingCounterfactualTreatment2019, melnychukCausalTransformerEstimating2022, bouchattaouiCausalContrastiveLearning2024} as a foundational building block to balance between treatment groups. More specifically, CRN \cite{bicaEstimatingCounterfactualTreatment2019} uses a gradient reversal layer (GRL) \cite{ganinDomainAdversarialTrainingNeural2016} between an LSTM representation network and a treatment balancing head. The treatment balancing head is optimized to predict the treatment domain from the representation, whereas the GRL encourages a treatment-invariant representation, forming an adversarial game. Similarly, CT \cite{melnychukCausalTransformerEstimating2022} replaces the LSTM backbone with a combination of three transformer sub-networks to capture long-range dependencies and proposes an adversarial game based on the domain confusion loss. However, an empirical study by \citet{huangEmpiricalExaminationBalancing2024} demonstrates that adversarial training results in a significant loss of covariate information, which affects the outcome prediction performance of CRN and CT. Representation invertibility is used as a theoretical tool to ensure a valid representation \cite{shalitEstimatingIndividualTreatment2017,johanssonGeneralizationBoundsRepresentation2022, melnychukBoundsRepresentationInducedConfounding2023}, which can help improve counterfactual estimation performance. CCPC \cite{bouchattaouiCausalContrastiveLearning2024} implicitly encourages reconstructable representation by maximizing mutual information between the input and representation following the InfoMax principle. Alternatively, Mamba-CDSP \cite{wangEffectiveEfficientTimeVarying2024} proposes an architecture-specific decorrelation loss to avoid the over-balancing problem of adversarial training. However, instead of matching the treatment conditional distribution, Mamba-CDSP is only equivalent to first-order moment matching of treatment regimes. 

\section{Problem Formulation}

\subsection{Outcome forecasting}

\begin{figure}[!ht]
    \centering
    \includegraphics[width=0.75\linewidth]{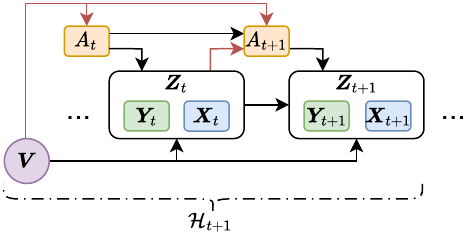}
    \caption{Causal graph for time-dependent confounding over $\mathcal{H}_t$.}
    \vspace{-5pt}
    \label{fig:time-dependent-confounding}
\end{figure}

Our method is designed in the context of the potential outcome framework \cite{rubinCausalInferenceUsing2005}. Further details on assumptions for causal identification are in \cref{sec:assumptions}. Consider a longitudinal dataset $\dataset{D}=\left\{\left\{A_{i,t}, \bm{Y}_{i,t}, \bm{X}_{i,t}\right\}_{t=1}^{T}, \bm{V}_{i}\right\}_{i=1}^N$ of $N$ units and $T$ time steps with \emph{categorical treatments} $A_{i,t}$, vector of \emph{continuous outcomes of interest} $\bm{Y}_{i,t}$, vector of \emph{continuous time-varying covariates} $\bm{X}_{i,t}$, and vector of \emph{static covariates} $\bm{V}_{i}$. Note that $\bm{Y}_{i,t}$ and $\bm{X}_{i,t}$ are disjoint covariates. Let $\bm{Z}_{i,t} = [\bm{X}_{i,t}, \bm{Y}_{i,t}]$ be the vector of all considered time-varying covariates. Since $\bm{X}_{i,t}$ may not be present for every dataset $\mathcal{D}$, $\bm{Z}_{i,t}$ is simply $\bm{Y}_{i,t}$ in that case. The history of a unit $i$ up to time $T_0$ is defined as $\mathcal{H}_{i,T_0}=\left\{\{A_{i,t}, \bm{Z}_{i,t}\}_{t=1}^{T_0}, \bm{V}_{i}\right\}$. The treatment $A_{i,t}$ affects the time-varying covariates $\bm{Z}_{i,t}$. \footnote{In several works, the treatment $A_{i,t-1}$ is denoted to affect the covariates $\bm{Z}_{i,t}$. We depart from this notation to ensure the history $\mathcal{H}_{i, t}$ has the same number of time steps for both the treatment $A$ and covariates $\bm{Z}$. This change helps simplify the notation in several places while also easing implementation.} For simplicity, we omit the unit index $i$ in the remainder of this article.

Consider the discrete treatment $A_{t} \in \{a^{(1)}, ..., a^{(K)}\}$ with $K$ possible values and denote the potential outcome under treatment $a^{(k)}$ as $\bm{Y}_{t}[a^{(k)}]$. Let the sequence of treatments from time $t+1$ to time $t+\tau$ as $A_{t+1:t+\tau}$. The task of interest is to estimate or forecast the outcome of the next $\tau$ time steps (prediction horizons) given the history up to time $T_0$ and the non-random sequence of treatment $a_{T_0+1:T_0+\tau}$:

{
\small
\begin{equation}
 \E [\bm{Y}_{T_0+\tau}[a_{T_0+1:T_0+\tau}] \mid \mathcal{H}_{T_0}].
\end{equation}
}

Practically, using a neural network, the history $\mathcal{H}_t$ can be encoded into a latent space by a \emph{representation network} $\Phi$ before being decoded into estimated outcomes by an \emph{outcome head} $F_{Y}$, similar to prior works \cite{shalitEstimatingIndividualTreatment2017, bicaEstimatingCounterfactualTreatment2019, melnychukCausalTransformerEstimating2022, bouchattaouiCausalContrastiveLearning2024}. Using the autoregressive decoding strategy, the estimation quantity can be expressed as:

{
\small
\begin{equation}
    F^Y(\Phi(\mathcal{H}_{t}), A_{t+1}) \approx \E [{\bm{Y}_{t+1} \mid \Phi(\mathcal{H}_{t}), A_{t+1}}].
\end{equation}
}

The existence of backdoor paths from static covariates $\bm{V}$ and current time-varying covariates $\bm{Z}_t$ to future treatment $A_{t+1}$, and the fact that current treatment $A_t$ also affects current time-varying covariates $\bm{Z}_t$ creates the time-dependent confounding. The confounding paths are shown in \cref{fig:time-dependent-confounding} in red. Under time-dependent confounding, the estimation of future outcomes is biased \cite{Robins2008EstimationOT}. 

\section{Method}

\subsection{Architecture overview}

Similarly to prior works, we parametrize the representation layer $\Phi$ by a sequential network such as LSTM (CRN) or transformer (CT). After the representation layer $\Phi$, we make an intuitive but important distinction in the architecture. In prior works, the learned representation is then forwarded to the respective treatment balancer $F^B$, or is concatenated with the next planned treatment $A_{t+1}$ to decode the next outcome $\bm{Y}_{t+1}$, as in \cref{fig:architecture}A. In this case, if $\Phi(\mathcal{H}_t)$ is high-dimensional, then the influence of $A_{t+1}$ is significantly reduced \citep{shalitEstimatingIndividualTreatment2017}. 

Instead of the treatment $A_{t+1}$ being concatenated with $\Phi(\mathcal{H}_t)$ at the outcome decoder $F^Y$, we consider the treatment as conditioning information that transforms the representation at the $F^C$ layer, as in \cref{fig:architecture}D. By explicitly modeling the treatment as conditioning information, more complex $F^C$ can be designed to modulate the representation while also enabling more flexible interactions that serve as building blocks for the rest of our method. Details on $F^C$ are deferred to \cref{sec:treatment-conditioning}.

\begin{figure*}[!ht]
    \centering
    \includegraphics[width=0.9\linewidth]{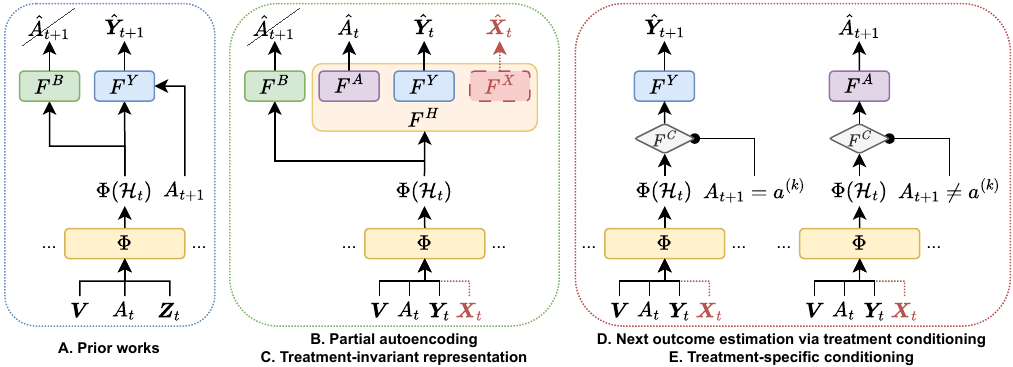}
    \caption{\textbf{A.} Prior works concatenate the representation $\Phi(\mathcal{H}_t)$ with treatment $A_{t+1}$ to predict next outcomes at $F^Y$. We make a distinction to model treatment $A_{t+1}$ as a transformation on the representation before being forwarded to respective heads. \textbf{B.} The history $\mathcal{H}_t$ is encoded into representation $\Phi(\mathcal{H}_t)$ before being forwarded to respective heads for decoding. The treatment, outcome, and time-varying covariates estimators $F^A$, $F^Y$ and $F^X$ reconstruct $A_{t}$, $\bm{Y}_t$ and $\bm{X}_t$. \textbf{C.} Simultaneously, balancing is applied to $\Phi(\mathcal{H}_t)$ by maximizing the entropy of the treatment balancer $F^B$. \textbf{D.} $A_{t+1}$-specific transformation is applied to the $\Phi(\mathcal{H}_t)$ before being decoded to the next outcomes by $F^Y$. \textbf{E.} The conditioning layer $F^C$ is encouraged to learn treatment-specific transformations of $\Phi(\mathcal{H}_t)$.}
    \vspace{-5pt}
    \label{fig:architecture}
\end{figure*}

\subsection{Input autoencoding}

We start with a base autoencoding network to satisfy the representation invertibility as in \cref{fig:architecture}B. At time step $T_0$, we encourage the decoding of current covariates $\{A_{T_0}, \bm{Z}_{T_0}\}$. Let $F^H$ be the time step decoding head, which includes three sub-heads: the \emph{treatment decoding head} $F^A$, the \emph{outcome of interest decoding head} $F^Y$, and the \emph{time-varying covariate decoding head} $F^X$ if $\bm{X}_t$ exists in the dataset. Note that while we encourage the decoding of the current time step $T_0$ for simplicity and efficiency, more decoding heads $F^H$ can be added to decode the last $\nu$ time steps. However, existing architectures for sequence have built-in mechanisms to selectively remove irrelevant information from the history, i.e., the forget gate in LSTM or the attention in transformer. Therefore, forcing decoding of many past inputs requires the network to remember those steps, which impairs the powerful learning ability of these mechanisms and can negatively impact performance \citep{mackayReversibleRecurrentNeural2018}. 

We define the reconstruction loss for categorical treatments $A_t$ as the cross-entropy, and for continuous outcomes of interest $\bm{Y}_t$ and continuous time-varying covariates $\bm{X}$ as the mean squared error. For $\mathcal{H}_t, A_t, \bm{Y}_t, \bm{X}_t, \sim \mathcal{D}$, we have:

{
\small
\begin{equation} \label{eq:loss-rect-treatment}
    \mathcal{L}^{RA}_t (\theta_{F^A}, \theta_{\Phi}) = - \sum_{j=1}^K \mathbb{I}_{[A_{t}=a^{(j)}]} \log F^A_j(\Phi(\mathcal{H}_t))
\end{equation}
\begin{equation} \label{eq:loss-rect-outcome}
    \mathcal{L}^{RY}_t (\theta_{F^Y}, \theta_{\Phi}) = \left\|F^Y(\Phi(\mathcal{H}_t)) - \bm{Y}_{t}\right\|_2^2
\end{equation}
\begin{equation} \label{eq:loss-rect-covariate}
    \mathcal{L}^{RX}_t (\theta_{F^X}, \theta_{\Phi}) = \left\|F^X(\Phi(\mathcal{H}_t)) - \bm{X}_{t}\right\|_2^2 
\end{equation}
}

where $\mathbb{I}_{\mathrm{C}}$ is the indicator function for condition $\mathrm{C}$. $\theta_{F^A}, \theta_{F^Y}, \theta_{F^X}, \theta_{\Phi}$ are the parameters of $F^A, F^Y$, $F^X$, and $\Phi$, accordingly.

The total reconstruction loss is then simply a combination of the above objectives with hyperparameters $\delta^A$ and $\delta^X$. If $\bm{X}_t$ is not present in the dataset, then $\delta^X=0$.

{
\small
\begin{equation} \label{eq:reconstruction}
    \begin{aligned}
        \mathcal{L}^{R}_t (\theta_{F^A}, \theta_{F^Y}, \theta_{F^X}, \theta_{\Phi}) &=  \mathcal{L}^{RY}_t (\theta_{F^Y}, \theta_{\Phi}) \\
        &+\delta^A \mathcal{L}^{RA}_t (\theta_{F^A}, \theta_{\Phi}) \\
        &+ \delta^X \mathcal{L}^{RX}_t (\theta_{F^X}, \theta_{\Phi}).
    \end{aligned}
\end{equation}
}

Prior empirical work by \citet{huangEmpiricalExaminationBalancing2024} demonstrates that adversarial training (introduced in \cref{sec:treatment-invariant}) is unstable and may cause the model to lose covariate information. Hence, the model risks violating representation invertibility \citep{shalitEstimatingIndividualTreatment2017, zhangLearningOverlappingRepresentations2020, johanssonGeneralizationBoundsRepresentation2022} and suffers loss of heterogeneity \cite{melnychukBoundsRepresentationInducedConfounding2023}. Despite being an important property to ensure causal identification, representation invertibility, in the context of longitudinal data, has rarely been enforced in practice (CRN - \citet{bicaEstimatingCounterfactualTreatment2019}; CT - \citet{melnychukCausalTransformerEstimating2022}) or only been enforced implicitly (CCPC - \citet{bouchattaouiCausalContrastiveLearning2024}). Our design starts from explicit partial invertibility as a foundational building block instead.

\subsection{Outcomes prediction via treatment conditioning} \label{sec:treatment-conditioning}

The task of interest is predicting the outcome for the next time step, which can be cast as applying a treatment-specific transformation on the representation $\Phi(\mathcal{H}_t)$ that corresponds to the future outcomes. This conditioning mechanism may be considered as an inductive bias, i.e, the changes in $\bm{Z}_{t+1}$ are caused by applying $A_{t+1}$.

More specifically, we apply an element-wise affine transformation to the representation, also known as the feature-wise linear modulation (FiLM) \citep{perezFiLMVisualReasoning2018}. For each treatment $A_{t+1}=a^{(i)}$, FiLM learns a scaling vector $\xi^{(i)}=R^\xi(a^{(i)})$ and bias vector $\beta^{(i)} = R^\beta(a^{(i)})$, where $R^\xi, R^\beta$ are the generators that produce the corresponding vectors. The conditioning layer $F^C$ can then be defined as:

{
\small
\begin{equation}
    F^C (\Phi(\mathcal{H}_t), a^{(i)}) = \Phi(\mathcal{H}_t) \odot \xi^{(i)} \oplus \beta^{(i)}
\end{equation}
}

where $\odot$ and $\oplus$ are element-wise multiplication and addition, respectively. Implementation-wise, $R^\xi$ and $R^\beta$ can be parameterized by a bag of embeddings for categorical treatment or a linear projection for continuous treatment.

Reusing the outcome regressor $F^Y$, we define the next outcome prediction loss for $\mathcal{H}_t, A_{t+1}, \bm{Y}_{t+1} \sim \mathcal{D}$ as:

{
\small
\begin{equation} \label{eq:loss-pred-outcome}
    \mathcal{L}^Y_t (\theta_{F^A}, \theta_{F^C}, \theta_{\Phi}) = \left\| F^Y (F^C (\Phi(\mathcal{H}_t), A_{t+1} )) - \bm{Y}_{t+1}\right\|_2^2
\end{equation}
}

Prior works combine the treatment $A_{t+1}$ with the representation $\Phi(\mathcal{H}_t)$ via  concatenation \citep{bicaEstimatingCounterfactualTreatment2019, melnychukCausalTransformerEstimating2022, bouchattaouiCausalContrastiveLearning2024}. After the next linear layer, this combination can be considered FiLM's special case, which is a learnable bias vector. However, this prior approach is limited in the expressiveness of the conditioning mechanism. Note that while we use an efficient FiLM mechanism, a more complex $F^C$ can be used. 

\subsection{Treatment-invariant representation learning} \label{sec:treatment-invariant}

Leveraging adversarial domain adaptation, recent methods learn a treatment-invariant representation $\Phi(\mathcal{H}_{t})$, breaking the association of history $\mathcal{H}_{t}$ and planned treatment $A_{t+1}$. Let $P^{\Phi}_{(j)}$ be the distribution of $\Phi(\mathcal{H}_{t})$ conditioned on a specific treatment $A_{t+1}=a^{(j)}$ for a fixed time step $t$, then recent methods aim to satisfy:

{
\small
\begin{equation}
    P^{\Phi}_{(1)} = P^{\Phi}_{(2)} = \dots = P^{\Phi}_{(K)}.      
    \label{eq:treatment-invariant}
\end{equation}
}

This aspect has been thoroughly explored in the literature. Let $P^A_{(j)}$ be the marginal distribution of $A_{t+1}=a^{(j)}$ for a fixed time step $t$, the simplified optimization objective for each method is equivalent to:

{
\small
\begin{equation}
    \begin{gathered}
        \min_{\Phi} \sum_{j=1}^{K} \frac{1}{K} \kldiv{P^{\Phi}_{(j)} \mid\mid \frac{1}{K}\sum_{k=1}^{K} P^{\Phi}_{(k)}} \\
        \text{\scriptsize CRN \citep[Appendix D]{bicaEstimatingCounterfactualTreatment2019}}
    \end{gathered}
    \label{eq:adv-crn}
\end{equation}
\begin{equation}
    \begin{gathered}
        \min_{\Phi} \sum_{j=1}^{K} \frac{1}{K} \kldiv{\sum_{k=1}^{K} P^A_{(k)} P^{\Phi}_{(k)} \mid\mid P^\Phi_{(j)}} \\ 
        \text{\scriptsize CT \citep[Appendix F]{melnychukCausalTransformerEstimating2022}} 
    \end{gathered}
    \label{eq:adv-ct}
\end{equation}
\begin{equation}
    \begin{gathered}
        \min_{\Phi} \sum_{j=1}^{K} P^A_{(j)}  \kldiv{P^{\Phi}_{(j)} \mid\mid \sum_{k=1}^{K} P^A_{(k)} P^{\Phi}_{(k)}} + D\\
        \text{\scriptsize CCPC \citep[Appendix G.6]{bouchattaouiCausalContrastiveLearning2024}} 
    \end{gathered}
    \label{eq:adv-ccpc}
\end{equation}
}

where \eqref{eq:adv-crn} is equivalent to minimizing a multi-distribution Jensen-Shannon divergence with equal weights and \eqref{eq:adv-ccpc} with weighting factors $P^A_{(1)}, ..., P^A_{(K)}$ \citep{linDivergenceMeasuresBased1991} plus an additional expected divergence term $D$. The minima for \eqref{eq:adv-crn}, \eqref{eq:adv-ct}, and \eqref{eq:adv-ccpc} satisfy \eqref{eq:treatment-invariant}. 

\textbf{Adversarial entropy maximization.} Similar to existing works, we aim to learn a representation that satisfies \eqref{eq:treatment-invariant}. Consider the \emph{balancing head} $F^B$, define the treatment classification loss as the cross-entropy for $\mathcal{H}_t, A_{t+1} \sim \mathcal{D}$:

{
\small
\begin{equation} \label{eq:loss-pred-domain}
    \mathcal{L}^{B}_t (\theta_{F^B}) = - \sum_{j=1}^K \mathbb{I}_{[A_{t+1}=a^{(j)}]} \log F^B_{j}(\Phi(\mathcal{H}_t)).
\end{equation}
}

While the balancing head $F^B$ predicts the confounded planned treatment $A_{t+1}$, the representation $\Phi(\mathcal{H}_t)$ confuses the balancing head by removing such information. To learn a treatment-invariant representation, we propose to maximize the prediction entropy of $F^B$ for $\mathcal{H}_t \sim \mathcal{D}$:

{
\small
\begin{equation} \label{eq:loss-entropy-max}
    \mathcal{L}^{E}_t (\theta_{\Phi}) = \sum_{j=1}^{K} F^B_j (\Phi(\mathcal{H}_t))  \log F^B_j(\Phi(\mathcal{H}_t)).
\end{equation}
}

Minimizing $\mathcal{L}^E_t$ is equivalent to maximizing the prediction entropy. For a fixed time step $t$, let $\bm{H}=\mathcal{H}_t$, we define the adversarial game as:

{
\small
\begin{equation} \label{eq:adv-game}
    \begin{gathered}
   \underset{F^B}{\arg\min} - \sum_{j=1}^K \E_{\Phi(\bm{H})|A=a^{(j)}}\left[\log  F^B_j(\Phi(\bm{H}))\right] P^A_{(j)} \\
    \underset{\Phi}{\arg \max} - \sum_{j=1}^K \E_{\Phi(\bm{H})}\left[F^B_j(\Phi(\bm{H})) \log F^B_j(\Phi(\bm{H}))\right] \\
    \text{subject to } \sum_{j=1}^{K} F^B_j(\Phi(\bm{H})) =1.
    \end{gathered}
\end{equation}
}

Next, we show that the equilibrium of \eqref{eq:adv-game} satisfies \eqref{eq:treatment-invariant}. In other words, the representation is treatment-invariant.

\begin{restatable}{theorem}{globalOptimumTheorem} \label{thm:global-optimum}
For a fixed $t$, there exists a pair $\Phi$ and $F^B$ that satisfies the equilibrium in \eqref{eq:adv-game}. The equilibrium holds if and only if $\Phi$ satisfies \eqref{eq:treatment-invariant}.
\end{restatable}

We also obtained the simplified optimization objective for \cref{eq:adv-game} as:

{
\small
\begin{equation}
    \min_{\Phi} \sum_{j=1}^{K} P^A_{(j)}  \kldiv{P^{\Phi}_{(j)} \mid\mid \sum_{k=1}^{K} P^A_{(k)} P^{\Phi}_{(k)}}.
\end{equation}
}

This is similar to the objective of CCPC in \cref{eq:adv-ccpc} but with a parsimonious and efficient implementation as entropy maximization instead of minimizing the CLUB \citep{chengCLUBContrastiveLogratio2020} mutual information upper bound.

Ideally, achieving the equilibrium condition of \eqref{eq:treatment-invariant} and satisfying representation invertibility results in a neural network with low estimation error for counterfactual outcomes. However, achieving the equilibrium of the adversarial game can be difficult in practice.

\textbf{Bound on prediction error.} Considering binary treatment $A\in \{0,1\}$, we can further show that the prediction error, under some assumptions, is bounded by the generalized Jensen-Shannon divergence between distributions inspired by the results from \citet{shalitEstimatingIndividualTreatment2017}. Let $G^Y = F^Y \circ F^C$. The expected factual, counterfactual, treated, and control losses for the loss function $\mathcal{L}$ are defined as:

{
\small
\begin{equation}
    \errF(G^Y, \Phi) = \E_{\bm{H}, \bm{Y}, A}\left[ \mathcal{L}(G^Y(\Phi(\bm{H}), A), \bm{Y}[A]) \right]
\end{equation}
\begin{equation}
    \errCF(G^Y, \Phi) = \E_{\bm{H}, \bm{Y}, A}\left[\mathcal{L}(G^Y(\Phi(\bm{H}), 1-A), \bm{Y}[1-A])\right]
\end{equation}
\begin{equation}
    \errF_{(1)}(G^Y, \Phi) = \E_{\bm{H}, \bm{Y}|A=1}\left[\mathcal{L}(G^Y(\Phi(\bm{H}), 1), \bm{Y}[1])\right]
\end{equation}
\begin{equation}
    \errF_{(0)}(G^Y, \Phi) = \E_{\bm{H}, \bm{Y}|A=0}\left[\mathcal{L}(G^Y(\Phi(\bm{H}), 0), \bm{Y}[0])\right].
\end{equation}
}

\begin{restatable}{theorem}{errorBoundTheorem} \label{thm:error-bound}
Let $\Phi$ be an invertible representation function and $G^Y$ be a hypothesis. Let $S=\sup_{\bm{H}, A} \Big| \E_{\bm{Y}|\bm{H}=\bm{h},A=a}\left[(\mathcal{L}\left(G^Y(\Phi(\bm{h}), a), \bm{Y}[a]\right)\right]\Big|$. Consider the weights $\pi^{(0)}, \pi^{(1)} > 0$ where $\pi^{(0)} + \pi^{(1)} = 1$, let $W=2S/(\sqrt{\pi^{(0)}} \pi^{(1)} + \pi^{(0)}\sqrt{\pi^{(1)}})$, we then have:

{
\small
\begin{equation}
\begin{aligned}
    \errF(G^Y, \Phi) + \errCF (G^Y, \Phi) & \leq  \errF_{(0)}(G^Y, \Phi)+\errF_{(1)}(G^Y, \Phi) 
    \ifinappendix
    \else
    \\ 
    \fi
    & + W \sqrt{\gjsdiv{\pi^{(0)}, \pi^{(1)}}{P^{\Phi}_{(0)} \mid\mid P^{\Phi}_{(1)}}}   .  
\end{aligned}
\end{equation}
}
\end{restatable}

It can be difficult to bound $S$ without additional assumptions. The first option is bounded loss functions. Alternatively, if we assume the domain of $\bm{V}$ and $\bm{Z}$ is bounded, then $S$ is bounded. In many applications, such as health care, we can typically assume naturally bounded distributions (e.g., heart rates, etc.). 

\subsection{Treatment-specific conditioning} \label{sec:treatment-specific}

By design, the conditioning layer $F^C$ should learn a treatment-specific transformation of the representation $\Phi(\mathcal{H}_t)$. However, during training, only the observable outcome corresponding to the factual treatment $A_{t+1}=a^{(k)}$ is optimized. Since the counterfactual outcomes cannot be observed, the conditioning with respect to remaining treatments $A_{t+1}\neq a^{(k)}$ does not receive any training signal. 

However, by leveraging the same strategy for next outcome prediction, we synthetically applied the conditioning on the representation $\Phi(\mathcal{H}_t)$ for all remaining treatments. Reusing the treatment classifier $F^A$, for treatment $a^{(c)} \neq a^{(k)}$, we optimize:

{
\small
\begin{equation} \label{eq:conditioning-loss}
    \begin{gathered}
        \mathcal{L}^{C}_{t, (c)} (\theta_{F^A}, \theta_{F^C},\theta_{\Phi}) \\ = - \sum_{j=1}^K \mathbb{I}_{[a^{(c)}=a^{(j)}]} \log F^A_{j}(F^C(\Phi(\mathcal{H}_t), a^{(c)})).
    \end{gathered}
\end{equation}
}

The final treatment-conditioning loss is simply the average of cross-entropy over all counterfactual treatments:

{
\small
\begin{equation} \label{eq:loss-pred-condtioning}
    \mathcal{L}_{t}^C (\theta_{F^A}, \theta_{F^C}, \theta_{\Phi}) =\frac{1}{K-1} \sum_{c \neq k} \mathcal{L}_{t, (c)}^C (\theta_{F^A}, \theta_{F^C}, \theta_{\Phi}).
\end{equation}
}

Combining the treatment-conditioning loss $\mathcal{L}^C_t$ with the next outcome prediction loss $\mathcal{L}_t^Y$, for every sample in every time step, all possible configurations of the conditioning layer $F^C$ are optimized. While the $\mathcal{L}^C_t$ loss is easier to optimize compared to the $\mathcal{L}_t^Y$ loss, we observe that the additional training signal helps improve counterfactual estimation. Further, during training, to prevent overfitting to the treatment-conditioning loss, it is beneficial to add label smoothing \cite{mullerWhenDoesLabel2019}, which is equivalent to replacing $\mathbb{I}_{[a^{(c)}=a^{(j)}]}$ with $\mathbb{I}_{[a^{(c)}=a^{(j)}]}(1-\alpha)+\mathbb{I}_{[a^{(c)}\neq a^{(j)}]}\frac{\alpha}{K}$ in \eqref{eq:conditioning-loss} for some hyperparameter $\alpha$.

\subsection{Autoregressive decoding}

\begin{figure}[!ht]
    \centering
    \includegraphics[width=0.75\linewidth]{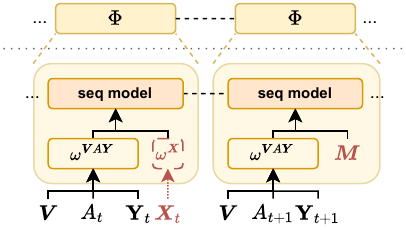}
    \caption{To handle input mismatch of the history $\mathcal{H}_{T_0}=\{\bm{V}, A_{t}, \bm{Y}_{t}, \bm{X}_{t}\}_{t=1}^{T_0}$ and autoregressively decoded sequences $\{\bm{V}, A_{t}, \bm{Y}_{t}\}_{t=T_0 +1}^{T_0+\tau}$, we replace the future time-varying covariates $\{\bm{X}\}_{T_0+1}^{T_0+\tau}$ with a learnable vector $\bm{M}$.}
    \vspace{-5pt}
    \label{fig:missing}
\end{figure}

At inference time, the predicted outcomes $\hat{\bm{Y}}_{t+1}$ are autoregressively reused to predict the next outcomes. However, if the dataset includes time-varying covariates $\bm{X}_t$, the network has a mismatched input dimension for future time steps as $\bm{X}_{t+1}$ is not available. For certain types of data, a straightforward solution is also learning to decode $\bm{X}_{t+1}$.

However, one of the potential problems with decoding $\bm{X}_t$ is that it can be difficult for many scenarios. Forcing the neural network to decode $\bm{X}_{t+1}$ in those cases may waste capacity on a less useful task. Existing works handle the problem by either using an encoder-decoder architecture \cite{limForecastingTreatmentResponses2018, bicaEstimatingCounterfactualTreatment2019, bouchattaouiCausalContrastiveLearning2024} or masking the corresponding attention of the decoder-only transformer \cite{melnychukCausalTransformerEstimating2022}. We introduce a dropout-like model-agnostic approach called temporal cutoff that can be applied to a decoder-only model as an alternative to the encoder-decoder architecture in prior works.

\begin{figure*}[!ht]
    \centering
    \includegraphics[width=\linewidth]{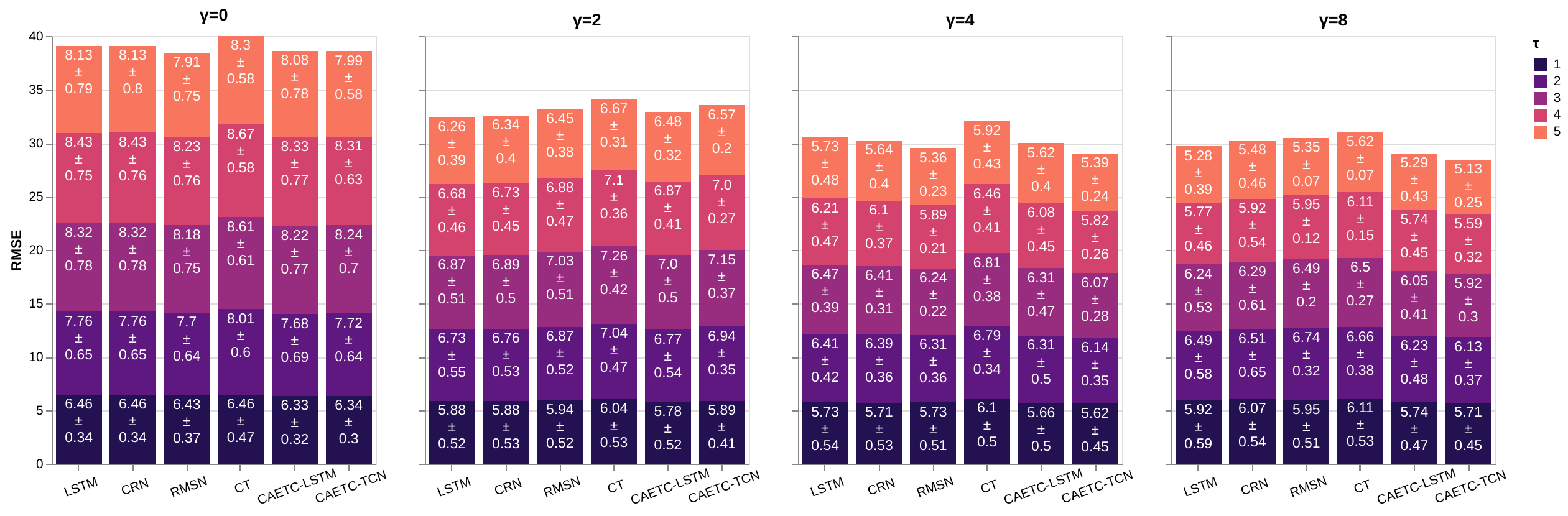}
    \caption{RMSEs for NSCLC fully synthetic data on random trajectories with increasing levels of time-dependent confounding on the training dataset. Each bar represents the RMSE for 5-step-ahead predictions. Lower is better.}
    \vspace{-5pt}
    \label{fig:nsclc-result-cf}
\end{figure*}
\begin{figure*}[!ht]
    \centering
    \includegraphics[width=\linewidth]{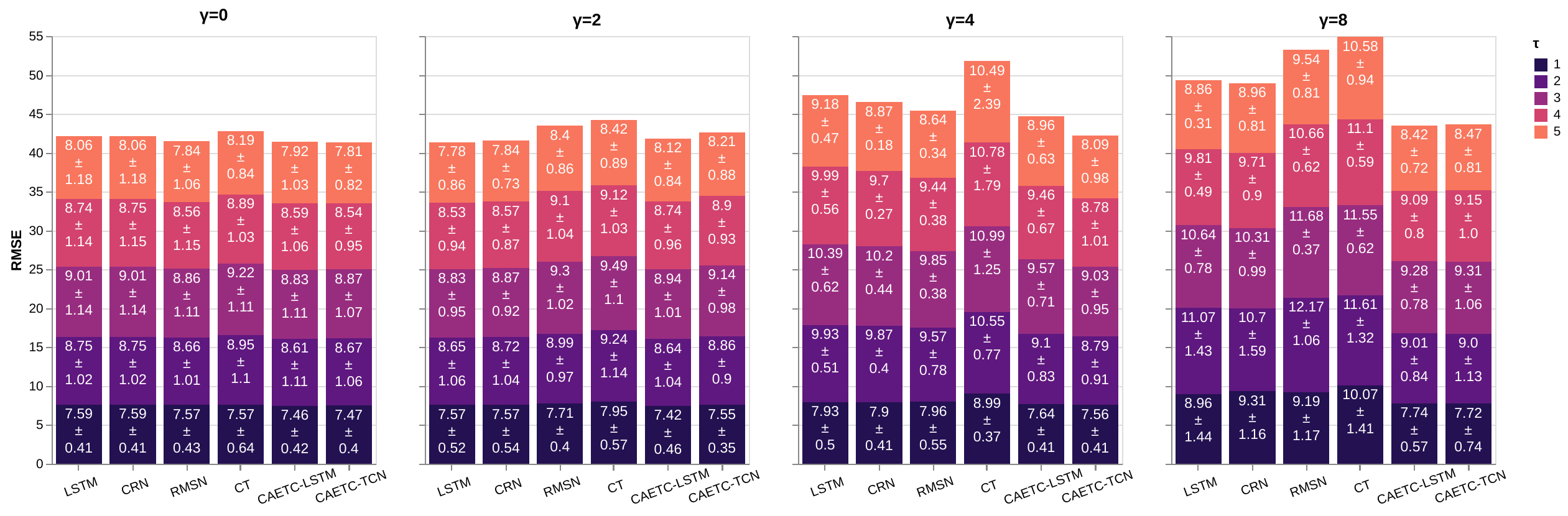}
    \caption{RMSEs for NSCLC fully synthetic data on no-confounding test set ($\gamma=0$) with increasing levels of time-dependent confounding on the training dataset. Training on $\gamma=0$ and testing on $\gamma=0$ represents the performance upper bound for each method. Due to time-dependent confounding, training on $\gamma \neq 0$ and testing on $\gamma = 0$ is expected to reduce performance. Lower is better.}
    \vspace{-5pt}
    \label{fig:nsclc-result-ood}
\end{figure*}

\textbf{Temporal cutoff.} For each training data point, during training, all timesteps $t \geq T_{\mathrm{cut}} \sim \mathrm{Uniform}(1, T)$ are dropped. To represent the value of the dropped timesteps, a learnable missingness vector $\bm{M}$ is used. We apply the temporal cutoff to only $\bm{X}$ if it exists. To allow more capacity to represent missingness, we first linearly project $\bm{X}_t \in \mathbb{R}^{u}$ to a larger dimension using $\omega^{\bm{X}}$ so that $\omega^{\bm{X}}(\bm{X}_t) \in \mathbb{R}^{v}$ where $v>u$. For non-cutoff time steps, $\omega^{\bm{X}}(\bm{X}_t) \in \mathbb{R}^{v}$ is used, while for cutoff time steps, the missingness vector $\bm{M} \in \mathbb{R}^{v}$ is used as the input to the sequence model as in \cref{fig:missing}. We also linearly project $[\bm{V}, A, \bm{Y}]$ using $\omega^{\bm{\bm{V}, A, \bm{Y}}}$ so that $\omega^{\bm{X}}(\bm{X}_t)$ do not dominate the input dimension. For simplicity, we let $\omega^{\bm{X}}$ and $\omega^{\bm{\bm{V}, A, \bm{Y}}}$ projects to the same dimension, i.e., $\omega^{\bm{\bm{V}, A, \bm{Y}}}([\bm{X}, A, \bm{Y}]) \in \mathbb{R}^v$.

\subsection{Overall training objective}

The overall training objective is:

{
\small
\begin{gather}
    \min \ \delta^E \mathcal{L}_{t}^B  (\theta_{F^B}) \\
    \begin{aligned}
        \min\ & \mathcal{L}_t^R(\theta_{F^A}, \theta_{F^Y}, \theta_{F^X}, \theta_\Phi) + \delta^{E} \mathcal{L}^E_t (\theta_{\Phi})\\
        + & \mathcal{L}^Y_t (\theta_{F^Y}, \theta_{F^C}, \theta_{\Phi}) 
        + \delta^A \mathcal{L}^C_t (\theta_{F^A}, \theta_{F^C}, \theta_{\Phi})         
    \end{aligned}
\end{gather}
}

with additional hyperparameters $\delta^E$. Note that we reuse the hyperparameters $\delta^A$ from \cref{eq:reconstruction} to reduce hyperparameter search efforts. 

\section{Experiments} \label{sec:exp}

\subsection{Baselines}
We evaluate the performance of \method{} against state-of-the-art methods for counterfactual estimation over time, including RMSN \citep{limForecastingTreatmentResponses2018}, CRN \citep{bicaEstimatingCounterfactualTreatment2019}, and CT \citep{melnychukCausalTransformerEstimating2022}. Since our method is model-agnostic, we implement two variants of {\method} using LSTM \citep{hochreiterLongShortTermMemory1997} and TCN \citep{baiEmpiricalEvaluationGeneric2018}, which we refer to as \method{}-LSTM and \method{}-TCN, respectively. We also provide a baseline using classical LSTM without balancing, which is equivalent to CRN with no adversarial training. Effective methods for counterfactual estimations should demonstrate performance improvement over classical LSTM. 

\begin{table*}[ht]
\centering
\caption{RMSEs on MIMIC-III semi-synthetic data with counterfactual random trajectories. Lower is better.}
\scriptsize
\begin{tabular}{c|c|ccccccccc|c} 
\toprule
& $\tau=1$ & $\tau=2$ & $\tau=3$ & $\tau=4$ & $\tau=5$ & $\tau=6$ & $\tau=7$ & $\tau=8$ & $\tau=9$ & $\tau=10$ & Avg \\ 
\midrule
\midrule
LSTM & 
\result{.363}{.028} &
\result{.417}{.025} &
\result{.463}{.021} &
\result{.506}{.020} &
\result{.546}{.017} &
\result{.581}{.016} &
\result{.615}{.016} &
\result{.649}{.015} &
\result{.682}{.016} &
\result{.715}{.017} &
.554 \\
\midrule
RMSN &
\result{.408}{.029} &
\result{.471}{.032} &
\result{.530}{.030} &
\result{.587}{.030} &
\result{.641}{.030} &
\result{.690}{.030} &
\result{.737}{.028} &
\result{.782}{.025} &
\result{.825}{.022} &
\result{.867}{.023} &
.654
\\
CRN & 
\result{.365}{.028} &
\result{.424}{.027} &
\result{.476}{.023} &
\result{.524}{.021} &
\result{.569}{.019} &
\result{.608}{.018} &
\result{.647}{.018} &
\result{.685}{.018} &
\result{.723}{.020} &
\result{.760}{.022} &
.578 \\
CT & 
\result{.324}{.017} &
\result{.411}{.024} &
\result{.499}{.035} &
\result{.582}{.040} &
\result{.660}{.041} &
\result{.733}{.043} &
\result{.808}{.046} &
\result{.882}{.046} &
\result{.941}{.048} &
\result{.999}{.053} &
.684 \\
\midrule
\methodcell{LSTM} & 
\result{.330}{.021} &
\result{.393}{.019} &
\result{.442}{.014} &
\result{.487}{.009} &
\result{.528}{.006} &
\result{.564}{.007} &
\result{.600}{.011} &
\result{.634}{.014} &
\result{.668}{.017} &
\result{.702}{.020} &
.535 \\
\midrule
\methodcell{TCN} & 
\result{.322}{.016} &
\result{.391}{.019} &
\result{.447}{.023} &
\result{.495}{.027} &
\result{.540}{.029} &
\result{.582}{.030} &
\result{.621}{.032} &
\result{.657}{.032} &
\result{.692}{.032} &
\result{.726}{.032} &
.547 \\
\bottomrule
\end{tabular}
\label{tab:mimic-synthetic-result}
\end{table*}

\subsection{Datasets}

A challenge in evaluating counterfactual estimation methods is the lack of counterfactual outcomes in real-world datasets. Therefore, we use synthetic data to obtain counterfactual ground truths. To ensure comprehensive evaluation, we benchmark the performance against standard benchmarks used in prior works \citep{limForecastingTreatmentResponses2018, bicaEstimatingCounterfactualTreatment2019, melnychukCausalTransformerEstimating2022}, including fully synthetic data based on non-small cell lung cancer simulation \cite{gengPredictionTreatmentResponse2017} and semi-synthetic data derived from the MIMIC-III dataset \cite{johnsonMIMICIIIFreelyAccessible2016}. We also evaluate \method{} on factual outcomes of real-world data from the MIMIC-III dataset to demonstrate the performance in practical usage. More details are available in the Appendix.

\subsection{Experiment with synthetic data} \label{sec:exp-nsclc}

The fully synthetic data is built upon the pharmacokinetics-pharmacodynamics model of non-small cell lung cancer (NSCLC) \citep{gengPredictionTreatmentResponse2017} under chemotherapy and radiotherapy treatment. Details of the simulation are in \cref{sec:nsclc}. We inject different levels of confounding during simulation, ranging from no confounding $\gamma=0$ to strong confounding $\gamma=8$. We set the maximum trajectory length $T$ to 60 and the prediction horizon $\tau$ to 5. We trained all methods on the trajectories of 10000 patients with a separate validation set of 1000 patients to select the best hyperparameters.  All methods are evaluated under two scenarios: random trajectories and no-confounding. Further details on the test setting are available in the \cref{sec:nsclc}.

\textbf{Random trajectories}. At every time step, there are $K$ possible outcomes. Therefore, at prediction horizon $\tau$, there are an exponential $K^{\tau}$ outcomes, making simulating every outcome prohibitively expensive. We randomly sample a fixed subset of $k$ outcomes for evaluation. The setting shares similarities with \citet{melnychukCausalTransformerEstimating2022}. We set $k=1$ and simulate for 1000 patients, resulting in 1000 test horizons per time step $t$. 

\textbf{No confounding}. As $\gamma$ can be controlled, we evaluate the model on a non-confounded ($\gamma=0$) test set to investigate whether the model has learned unbiased estimation. This setting is more difficult than random trajectories, as not only do the prediction horizons not follow the confounded trajectories, but the input history also does not resemble the confounded training set. The setting is similar to \citet{limForecastingTreatmentResponses2018}. We use the trajectories of 1000 patients.  

\textbf{Results.} \method-LSTM and \method-TCN demonstrate strong performance with increasing time-dependent confounding as in \cref{fig:nsclc-result-cf} and \cref{fig:nsclc-result-ood}. For all methods, training and testing on non-confounded datasets $\gamma=0$ in \cref{fig:nsclc-result-ood} results in a very small difference in performance, which is desirable. Due to the limited amount of data, similar to real-world scenarios, LSTM and TCN architectures exhibit good performance compared to the transformer architecture. For small confounding, regular LSTM is more stable and demonstrates strong performance compared to counterfactual methods. However, as the confounding factor increases, only \method-LSTM and \method-TCN exhibit stable counterfactual estimation performance. Prior works (CRN, CT) do not outperform a vanilla LSTM with empirical risk minimization due to covariate information loss (CRN, CT), which is consistent with prior literature \citep{huangEmpiricalExaminationBalancing2024}. Our method uses a partial-autoencoding architecture, which is less likely to suffer from the same issues. Furthermore, it can be observed that under the more challenging setting of \cref{fig:nsclc-result-ood}, {\method} demonstrate strong improvement over existing works.

\subsection{Experiment with semi-synthetic data}

We employ a semi-synthetic dataset generated based on the MIMIC-III \cite{johnsonMIMICIIIFreelyAccessible2016} intensive care units dataset. Details of the simulation are in \cref{sec:mimic-synthetic}.  We set the maximum trajectory length $T$ to 100 and the prediction horizon $\tau$ to 10. All methods are trained and validated on 2000 and 200 trajectories, respectively. We evaluate the method on random trajectories ($k=1$) similar to \cref{sec:exp-nsclc}, which includes 200 patients, resulting in 200 test horizons per time step $t$.

\textbf{Results.} Similar to the synthetic experiment, \method-LSTM and \method-TCN demonstrate strong improvement across all time steps as in \cref{tab:mimic-synthetic-result}.

\subsection{Experiment with real-world data}

We further evaluate {\method} on real-world data based on the MIMIC-III dataset \cite{johnsonMIMICIIIFreelyAccessible2016}. However, for real-world data, counterfactual outcomes are not available. Nonetheless, based on \cref{thm:error-bound}, we can expect that a smaller sum of treated and control error indicates a tighter error bound. Therefore, the performance on observable outcomes is still a useful evaluation metric. We set the maximum trajectory length $T$ to 60 and the prediction horizon $\tau$ to 5. All methods are trained, validated, and tested on 5000, 500, and 500 trajectories sampled from the dataset, respectively.

\textbf{Results.} Similar to previous experiments, \method-LSTM and \method-TCN achieve strong improvement compared to existing methods as in \cref{tab:mimic-real}.

{\setlength{\tabcolsep}{3pt}
\begin{table}[!ht]
\centering
\caption{RMSEs on MIMIC-III real-world data. Lower is better.}
\scriptsize
\begin{tabular}{c|c|cccc|c}
\toprule
& $\tau=1$ & $\tau=2$ & $\tau=3$ & $\tau=4$ & $\tau=5$ & Avg \\
\midrule
\midrule
LSTM & 
\result{6.80}{0.02} &
\result{7.32}{0.05} &
\result{7.61}{0.05} &
\result{7.82}{0.05} &
\result{7.99}{0.07} &
7.51 \\
\midrule
RMSN & 
\result{7.29}{0.07} &
\result{7.80}{0.07} &
\result{8.10}{0.10} &
\result{8.33}{0.16} &
\result{8.54}{0.21} &
8.01 \\
CRN & 
\result{6.81}{0.04} &
\result{7.35}{0.05} &
\result{7.64}{0.06} &
\result{7.86}{0.07} &
\result{8.04}{0.10} &
7.54 \\
CT &
\result{6.70}{0.04} &
\result{7.37}{0.06} &
\result{7.71}{0.07} &
\result{7.93}{0.07} &
\result{8.10}{0.04} &
7.56\\
\midrule
\methodcell{LSTM} &
\result{6.72}{0.05} &
\result{7.27}{0.05} &
\result{7.56}{0.06} &
\result{7.76}{0.06} &
\result{7.91}{0.08} &
7.45 \\
\midrule
\methodcell{TCN} & 
\result{6.69}{0.06} &
\result{7.24}{0.05} &
\result{7.53}{0.04} &
\result{7.72}{0.03} &
\result{7.88}{0.06} &
7.41 \\
\bottomrule
\end{tabular}
\label{tab:mimic-real}
\end{table}
}

\subsection{Ablation}

We further examined the performance contribution of each component in {\method} using an LSTM backbone. We consider the following variations: (1) \method-LSTM, which achieves state-of-the-art performance; (2) \method-LSTM ($\delta^C=0$), which removes treatment conditioning loss $\mathcal{L}^C$ (\cref{sec:treatment-specific}); and (3) \method-LSTM ($\delta^C=\delta^E=0$), which further removes adversarial entropy maximization (\cref{sec:treatment-invariant}). Variation (3) is the basic version, which is only a partial-autoencoding network with treatment conditioning.

\textbf{Results.} Observe that each loss contributes to the overall performance. \method-LSTM ($\delta^C = \delta^E = 0$) is more biased over the prediction horizon compared to adversarially-trained \method-LSTM ($\delta^C=0$). Due to the partial-autoencoding, {\method} can balance the representation with less covariate information loss. 

{\setlength{\tabcolsep}{3pt}
\begin{table}[!ht]
\centering
\caption{Ablation with the same setting of \cref{fig:nsclc-result-ood} with $\gamma=8$.}
\scriptsize
\begin{tabular}{c|ccccc|c}
\toprule
\methodcell{LSTM} & $\tau=1$ & $\tau=2$ & $\tau=3$ & $\tau=4$ & $\tau=5$ & Avg \\ 
\midrule
\midrule
Full & 
\result{7.74}{0.57} &
\result{9.01}{0.84} &
\result{9.28}{0.78} &
\result{9.09}{0.8} &
\result{8.42}{0.72} &
8.71 \\
\midrule
$\delta^C=0$ & 
\result{7.75}{0.57} &
\result{9.09}{0.86} &
\result{9.45}{0.87} &
\result{9.37}{0.98} &
\result{8.79}{1.04} &
8.89    \\
\midrule
\makecell{$\delta^C=0$ \vspace{-2pt} \\ $\delta^{E}=0$} & 
\result{8.01}{0.91} &
\result{9.58}{1.31} &
\result{10.06}{1.44} &
\result{10.05}{1.61} &
\result{9.45}{1.49} &
9.43 \\
\bottomrule
\end{tabular}
\label{tab:ablation}
\end{table}
}

\section{Conclusion}

We introduced a novel method for counterfactual estimation over time called {\methodlower}. {\method} utilizes a partial autoencoding architecture for representation invertibility and casts outcome prediction as treatment conditioning on the adversarially balanced representation. Extensive empirical experiments show that {\method} improves significantly over baselines. 

\section*{Acknowledgements}

This work is supported by a VinUni-Illinois Smart Health Center grant (PI: Lav R. Varshney). Nghia D. Nguyen is supported by the Vingroup Science and Technology Scholarship Program for Overseas Study for Master’s and Doctoral Degrees.

\section*{Impact Statement}

This paper proposes a novel method for counterfactual estimation over time, which can be beneficial for sequential decision-making systems in domains such as healthcare, economics, or public policy. Improved counterfactual and treatment effect estimation can reduce the bias when learning from observational data. However, misuse or incorrect estimates can result in potentially harmful outcomes, so careful evaluation is required in sensitive applications.

\clearpage

\bibliography{references}
\bibliographystyle{icml2026}

\newpage

\appendix
\onecolumn

\section{Assumptions for causal identification} \label{sec:assumptions}

\subsection{Assumptions}

We make the standard assumptions for causal identification \citep{Robins2008EstimationOT} that are also used in related methods \citep{bicaEstimatingCounterfactualTreatment2019, melnychukCausalTransformerEstimating2022, bouchattaouiCausalContrastiveLearning2024}.

\textbf{Consistency}: If $A_{\leq t}=a_{\leq t}$ is a sequence of treatments for a given unit, then the potential outcome is equivalent to the observed outcome or $\bm{Y}_{t}[a_{\leq t}]=\bm{Y}_{t}$.

\textbf{Sequential Positivity}: For every time step, there is a strictly positive probability to observe any treatment level, given a realization of the history. Formally, if $P(\mathcal{H}_t = h_t)>0$, then $0 < P(A_{t+1} = a_{t+1}| \mathcal{H}_t = h_{t})$ for all $a_{t+1}$.

\textbf{Sequential Ignorability}: For every time step, the treatment assignment $A_{t+1}$ is independent of potential outcomes, conditioned on the history, or $\bm{Y}_{\geq t+1}[a_{\geq t+1}] \perp A_{t+1} | \mathcal{H}_t$ for all $a_{\geq t+1}$.

\section{Proof of \cref{thm:global-optimum}}

The following proofs are adapted from CRN \citep{bicaEstimatingCounterfactualTreatment2019} and CT \citep{melnychukCausalTransformerEstimating2022}.

\begin{lemma} \label{lem:optima-ga}
For a fixed representation network $\Phi$, let $\bm{R}=\Phi(\bm{H})$, then the optimal predictor $F^B$ is:
{
\small
\begin{equation} \label{eq:optimal-ga}
    F^{B*}_{j}(\bm{R}) = \frac{P^A_{(j)} P^\Phi_{(j)}(\bm{R})}{\sum_{k=1}^K P^A_{(k)} P^\Phi_{(k)}(\bm{R})}.
\end{equation}
}
\end{lemma}
\begin{proof}
The objective for the predictor $F^B$ in \eqref{eq:adv-game} is:

{
\small
\begin{equation}
    F^{B*} = \arg \min_{F^B} - \sum_{j=1}^K \int \log \left( F^B_j(\bm{r})\right)P^\Phi_{(j)}(\bm{r}) P^A_{(j)} d\bm{r} \quad \text{subject to} \quad \sum_{j=1}^K F^B_j(\bm{r}) =1.
\end{equation}
}

Minimizing the objective point-wise and applying the Lagrange multiplier:

{
\small
\begin{equation}
    F^{B*} = \arg \min_{F^B} - \sum_{j=1}^K \log \left( F^B_j(\bm{r})\right)P^\Phi_{(j)}(\bm{r}) P^A_{(j)} + \lambda \left(\sum_{j=1}^K F^B_j(\bm{r}) -1\right).
\end{equation}
}

Taking the derivative with respect to $F^B_j(\bm{r})$ and setting to 0:

{
\small
\begin{equation}
    F^{B*}_j (\bm{r}) = \frac{ P^A_{(j)}P^\Phi_{(j)}(\bm{r})}{\lambda}.
\end{equation}
}

Solve for $\lambda$ by the constraint $\sum_{j=1}^K F^B_j(\bm{r}) =1$ to obtain \eqref{eq:optimal-ga}.
\end{proof}

\globalOptimumTheorem*

\begin{proof}
The objective for representation network $\Phi$ in \eqref{eq:adv-game} with $P^{\Phi}(\bm{r})=P(\Phi(\bm{h}))$ is
{
\small
\begin{equation}
    \Phi^{*} = \arg \min_{\Phi} \sum_{j=1}^K \int  F^B_j(\bm{r}) \log \left( F^B_j(\bm{r}) \right) P^{\Phi}(\bm{r}) d\bm{r}.
\end{equation}
}
Use the optimal predictor obtained from \cref{lem:optima-ga} and $P^{\Phi}(\bm{r}) = \sum_{j=1}^{K} P^A_{(j)} P^\Phi_{(j)}(\bm{r})$:

{
\begingroup
\allowdisplaybreaks
\small
\begin{align}
    \Phi^{*} &= \arg \min_{\Phi} \sum_{j=1}^K \int  \frac{P^A_{(j)} P^\Phi_{(j)}(\bm{r})}{\sum_{k=1}^K P^A_{(k)} P^\Phi_{(k)}(\bm{r})} \log \left(\frac{P^A_{(j)} P^\Phi_{(j)}(\bm{r})}{\sum_{k=1}^K P^A_{(k)} P^\Phi_{(k)}(\bm{r})} \right) \sum_{k=1}^{K} P^A_{(k)} P^\Phi_{(k)}(\bm{r}) d\bm{r} \\
    &= \arg \min_{\Phi} \sum_{j=1}^K \int  P^A_{(j)} P^\Phi_{(j)}(\bm{r}) \log \frac{P^A_{(j)} P^\Phi_{(j)}(\bm{r})}{\sum_{k=1}^K P^A_{(k)} P^\Phi_{(k)}(\bm{r})} d\bm{r} \\
    &= \arg \min_{\Phi} \left( \sum_{j=1}^K  \int  P^A_{(j)} P^\Phi_{(j)}(\bm{r}) \log \frac{P^\Phi_{(j)}(\bm{r})}{\sum_{k=1}^K P^A_{(k)} P^\Phi_{(k)}(\bm{r})} d\bm{r}  + \sum_{j=1}^K  \int  P^A_{(j)} P^\Phi_{(j)}(\bm{r} )  \log P^A_{(j)} d\bm{r} \right) \\
    &= \arg \min_{\Phi} \left( \sum_{j=1}^K P^A_{(j)} \int   P^\Phi_{(j)}(\bm{r}) \log \frac{P^\Phi_{(j)}(\bm{r})}{\sum_{k=1}^K P^A_{(k)} P^\Phi_{(k)}(\bm{r})} d\bm{r}  + \underbrace{\sum_{j=1}^K P^A_{(j)} \log P^A_{(j)} \underbrace{\int  P^\Phi_{(j)}(\bm{r}) d\bm{r}}_{=1}}_{=C} \right) \\
    &= \arg \min_{\Phi} \sum_{j=1}^K P^A_{(j)} \int   P^\Phi_{(j)}(\bm{r}) \log \frac{P^\Phi_{(j)}(\bm{r})}{\sum_{k=1}^K P^A_{(k)} P^\Phi_{(k)}(\bm{r})} d\bm{r} \\
    &= \arg \min_{\Phi} \sum_{j=1}^K P^A_{(j)} \kldiv{P^\Phi_{(j)} \mid\mid \sum_{k=1}^K P^A_{(k)} P^\Phi_{(k)}} \\
    &= \arg \min_{\Phi} \gjsdiv{P^A_{(0)}, ..., P^A_{(k)}}{P^\Phi_{(0)}, ..., P^\Phi_{(k)}}
\end{align}
\endgroup
}

where $\gjsdiv{P^A_0, ..., P^A_{(k)}}{P^\Phi_{(0)}, ..., P^\Phi_{(k)}}$ is the generalized Jensen-Shannon divergence with weight $P^A_{(0)}, ..., P^A_{(k)}$ for distributions $P^\Phi_{(0)}, ..., P^\Phi_{(k)}$. The divergence is non-negative and equals zero if and only if all distributions are identical, which satisfy \eqref{eq:treatment-invariant}.

\end{proof}

\section{Proof of \cref{thm:error-bound}}

We restate \citet[Lemma A4]{shalitEstimatingIndividualTreatment2017} with a specific IPM metric, which is the total variation distance between two distributions $P$ and $Q$, denoted as $\tvdist{P \mid\mid Q} = \frac{1}{2}\int |P(\bm{x}) - Q(\bm{x})| d\bm{x}$. 

\begin{lemma} \label{lem:err-cf}
Let $\Phi: \domain{\bm{H}} \to \domain{\bm{R}}$ be an invertible representation. Consider the supremum $S=\sup_{\bm{H}, A} \Big|\E_{\bm{Y}|\bm{H}=\bm{h},A=a}\left[(\mathcal{L}(G^Y(\Phi(\bm{h}), a), \bm{Y}[a])\right] \Big|$, let $u = P^A_{(1)}$, we then have:

{
\small
\begin{equation}
    \errCF(G^Y, \Phi) \leq (1-u) \errF_{(1)}(G^Y, \Phi) + u \errF_{(0)} (G^Y, \Phi) + 2 S \cdot \tvdist{P^{\Phi}_{(0)} \mid\mid P^{\Phi}_{(1)}}.
\end{equation}
}

\end{lemma}

\begin{proof}

Let:
{
\small
\begin{equation}
        \errCF_{(1)}(G^Y, \Phi) = \E_{\bm{H}, \bm{Y}|A=1}\left[\mathcal{L}(G^Y(\Phi(\bm{H}),0), \bm{Y}[0])\right] \quad \mathrm{and} \quad \errCF_{(0)}(G^Y, \Phi) = \E_{\bm{H}, \bm{Y}|A=0}\left[\mathcal{L}(G^Y(\Phi(\bm{H}),1), \bm{Y[1]})\right].
\end{equation}
}

From \citet[Lemma A3]{shalitEstimatingIndividualTreatment2017}, then

{
\small
\begin{equation}
    \errF(G^Y, \Phi) = u \errF_{(1)} (G^Y, \Phi) + (1-u) \errF_{(0)}(G^Y, \Phi)
    \quad\text{and}\quad
    \errCF(G^Y, \Phi) = (1-u) \errCF_{(1)} (G^Y, \Phi) + u\errCF_{(0)}(G^Y, \Phi).
\end{equation}
}

Let $P_{(a)}(\bm{h}) = P(\bm{h}|A=a)$ and $\bar{\mathcal{L}}^{G^Y, \Phi}(\bm{h}, a) = \E_{\bm{Y}|\bm{H}=\bm{h},A=a}[\mathcal{L}(G^Y(\Phi(\bm{h}),a), \bm{Y}[a])]$. We also have that $S=\sup_{\bm{H}, A} \big|  \bar{\mathcal{L}}^{G^Y, \Phi}(\bm{h}, a) \big|$. Then:

{
\begingroup
\allowdisplaybreaks
\small
\begin{align}
    &\errCF(G^Y, \Phi) - \big[(1-u)\errF_{(1)} (G^Y, \Phi) + u \errF_{(0)} (G^Y, \Phi)\big] \\
    =& \big[(1-u) \errCF_{(1)} (G^Y, \Phi) + u\errCF_{(0)}(G^Y, \Phi)\big] -  \big[(1-u)\errF_{(1)} (G^Y, \Phi) + u \errF_{(0)} (G^Y, \Phi)\big] \\
    =& (1-u) \big[\errCF_{(1)} (G^Y, \Phi) - \errF_{(1)}(G^Y, \Phi)\big] + u \big[\errCF_{(0)} (G^Y, \Phi) - \errF_{(0)} (G^Y, \Phi)\big] \\
    =& (1-u) \int_{\domain{\bm{H}}} \bar{\mathcal{L}}^{G^Y, \Phi}(\bm{h}, 1) \big(P_{(0)}(\bm{h}) - P_{(1)}(\bm{h})\big) d\bm{h} +u \int_{\domain{\bm{H}}} \bar{\mathcal{L}}^{G^Y, \Phi}(\bm{h}, 0) \big(P_{(1)}(\bm{h}) - P_{(0)}(\bm{h})\big) d\bm{h} \\ 
    \leq& (1-u) \sup_{\bm{H}} \big|\bar{\mathcal{L}}^{G^Y, \Phi}(\bm{h}, 1) \big| \int_{\domain{\bm{H}}} \big|P_{(0)}(\bm{h}) - P_{(1)}(\bm{h})\big| d\bm{h} +u \sup_{\bm{H}} \big|\bar{\mathcal{L}}^{G^Y, \Phi}(\bm{h}, 0) \big| \int_{\domain{\bm{H}}} \big|P_{(1)}(\bm{h}) - P_{(0)}(\bm{h})\big| d\bm{h} \label{eq:supbound} \\ 
    \leq & (1-u) S \int_{\domain{\bm{H}}} \big|P_{(0)}(\bm{h}) - P_{(1)}(\bm{h})\big| d\bm{h} +u S \int_{\domain{\bm{H}}} \big|P_{(1)}(\bm{h}) - P_{(0)}(\bm{h})\big| d\bm{h} \\ 
    = & S \int_{\domain{\bm{H}}} \big|P_{(0)}(\bm{h}) - P_{(1)}(\bm{h})\big| d\bm{h} \\
    = & S \int_{\domain{\bm{R}}} \big|P^\Phi_{(0)}(\bm{r}) - P^\Phi_{(1)}(\bm{r})\big| d\bm{r} \label{eq:change-var} \\ 
    =& 2S \tvdist{P^{\Phi}_{(0)} \mid\mid P^{\Phi}_{(1)}}
\end{align}
\endgroup
}
where \eqref{eq:supbound} is by the H\"{o}lder's inequality and \eqref{eq:change-var} is by the change of variables formula.
\end{proof}

\begin{lemma} \label{lem:gjs}
Consider the generalized Jensen-Shannon divergence between two distributions $P$ and $Q$ with weighting factors $\pi^P, \pi^Q > 0$ so that $\pi^P+\pi^Q = 1$, then:

{
\small
\begin{equation}
\tvdist{P \mid\mid Q} \leq \frac{\sqrt{\gjsdiv{\pi^P, \pi^Q}{P \mid\mid Q}}}{\sqrt{\pi^P}\pi^Q + \pi^P \sqrt{\pi^Q}}.
\end{equation}
}

\end{lemma}

\begin{proof}

{
\small
\begin{align}
    \tvdist{P \mid\mid \pi^P P+\pi^Q Q} &=\frac{1}{2} \int |P (x)-\pi^P P(x)-\pi^Q Q(x)|dx \\
    &= \frac{1}{2} \int  |\pi^Q P(x) - \pi^Q Q(x)|dx \\
    &= \pi^Q \frac{1}{2}\int  |P (x) - Q(x)|dx \\
    &= \pi^Q \tvdist{P \mid\mid Q}.
\end{align}
}

Similarly, we also have $\tvdist{Q \mid\mid \pi^P P+\pi^Q Q}= \pi^P \tvdist{P \mid\mid Q}$. Then

{
\small
\begin{equation}
\begin{aligned}
    \sqrt{\pi^P} \pi^Q \tvdist{P \mid\mid Q} +  \pi^P \sqrt{\pi^Q} \tvdist{P \mid\mid Q} = \sqrt{\pi^P} \tvdist{P \mid\mid \pi^P P+\pi^Q Q} + \sqrt{\pi^Q}\tvdist{Q \mid\mid \pi^P P+\pi^Q Q}.
\end{aligned}
\end{equation}
}

Using Pinsker's inequality \cite{tsybakovLowerBoundsMinimax2009}, $\tvdist{P \mid\mid Q} \leq \sqrt{\kldiv{P \mid\mid Q}/2}$ and the inequality $a+b \leq \sqrt{2a^2+2b^2}$:

{
\small
\begin{align}
    & \sqrt{\pi^P} \tvdist{P \mid\mid \pi^P P+\pi^Q Q} + \sqrt{\pi^Q} \tvdist{Q \mid\mid \pi^P P+\pi^Q Q} \\
    &\leq \sqrt{\pi^P} \sqrt{\frac{1}{2} \kldiv{P \mid\mid \pi^P P+\pi^Q Q}} + \sqrt{\pi^Q} \sqrt{\frac{1}{2} \kldiv{Q \mid\mid \pi^P P+\pi^Q Q}} \\
    &\leq \sqrt{ 2 \pi^P \frac{1}{2} \kldiv{P \mid\mid \pi^P P+\pi^Q Q}+ 2\pi^Q \frac{1}{2} \kldiv{Q \mid\mid \pi^P P+\pi^Q Q}} \\
    &= \sqrt{ \pi^P  \kldiv{P \mid\mid \pi^P P+\pi^Q Q}+ \pi^Q \kldiv{Q \mid\mid \pi^P P+\pi^Q Q}} \\
    &= \sqrt{\gjsdiv{\pi^P, \pi^Q}{P \mid\mid Q}}.
\end{align}
}

It follows that
{
\small
\begin{equation}
     \tvdist{P \mid\mid Q} \leq \frac{\sqrt{\gjsdiv{\pi^P, \pi^Q}{P\mid\mid Q}}}{\sqrt{\pi^P} \pi^Q + \pi^P \sqrt{\pi^Q}}.
\end{equation}
}

Equality occurs when $P$ and $Q$ are identical.
\end{proof}

\errorBoundTheorem*

\begin{proof}

By applying Lemma A3 of \citet{shalitEstimatingIndividualTreatment2017}, \cref{lem:err-cf}, and \cref{lem:gjs}:

{
\small
\begin{align}
    & \errF(G^Y, \Phi) + \errCF(G^Y, \Phi) \\
    &\leq u \errF_{(1)} (G^Y, \Phi) + (1-u) \errF_{(0)}(G^Y, \Phi) + (1-u) \errF_{(1)}(G^Y, \Phi) + u \errF_{(0)} (G^Y, \Phi) + 2S \cdot \tvdist{P^{\Phi}_{(0)} \mid\mid P^{\Phi}_{(1)}} \\
    &= \errF_{(0)}(G^Y, \Phi)+\errF_{(1)} (G^Y, \Phi) +2 S \cdot \tvdist{P^{\Phi}_{(0)} \mid\mid P^{\Phi}_{(1)}} \\
    &\leq  \errF_{(0)}(G^Y, \Phi)+\errF_{(1)} (G^Y, \Phi)  + 2S \frac{\sqrt{\gjsdiv{\pi^{(0)}, \pi^{(1)}}{P\mid\mid Q}}}{\sqrt{\pi^{(0)}} \pi^{(1)} + \pi^{(0)} \sqrt{\pi^{(1)}}}.
\end{align}
}
\end{proof}

\section{Non-small Cell Lung Cancer PK-PD Simulation} \label{sec:nsclc}

\subsection{Simulation details} \label{sec:nsclc-sim-details}

Similar to prior works \citep{limForecastingTreatmentResponses2018, bicaEstimatingCounterfactualTreatment2019, melnychukCausalTransformerEstimating2022}, we evaluate \method{} on a pharmacokinetics-pharmacodynamics model of non-small cell lung cancer \citep{gengPredictionTreatmentResponse2017}. 

The model simulates the tumor volume over time under chemotherapy and radiotherapy. Parameters in the form of $\zeta^{\text{param}}$ are either given or sampled from the distribution in \citet{gengPredictionTreatmentResponse2017}. Consider the discrete-time model of volume of the tumor $Y_t$, the chemotherapy drug concentration $C_t$, and the radiotherapy dose $D_t$:

{
\small
\begin{equation}
Y_{t+1} = \left(\underbrace{1+ \zeta^\rho \log \frac{\zeta^K}{Y_t}}_{\text{tumor growth}} + \underbrace{\epsilon_t}_{\text{noise}} - \underbrace{\zeta^{\beta_c} C_{t+1}}_{\text{chemotherapy}} - \underbrace{\left(\zeta^\alpha D_{t+1} + \zeta^\beta D_{t+1}^2\right)}_{\text{radiotherapy}}\right) Y_t
\end{equation}
}

where one unit of time is one day. A noise term $\epsilon_t \sim \mathrm{Normal}(0, 0.01^2)$ is added to account for randomness in tumor growth. The tumor is assumed to be spherical.

The chemotherapy concentration is modeled as an exponential decay with a half-life of one day as $C_t = \tilde{C}_t + C_{t-1} / 2$. The dosage if used of chemotherapy $\tilde{C}_t$ and radiotherapy $D_t$ are $5.0$ mg/m$^3$ of Vinblastine and $2.0$ Gy fractions, respectively. Heterogeneous static features $V$ are generated by augmenting the prior means of $\zeta^\alpha, \zeta^\beta$. The patients are divided into three groups $V \sim \mathrm{Uniform}\{1,2,3\}$. The augmented prior means $\zeta^{\mu_{\alpha}}, \zeta^{\mu_{\beta_C}}$ are then defined as:

{
\small
\begin{equation}
    \zeta^{\mu_{\beta_C}'} = \left\{\begin{array}{cc}
         1.1 \zeta^{\mu_{\beta_C}} & \text{if } V=3 \\
         \zeta^{\mu_{\beta_C}} & \text{otherwise}
    \end{array}\right.
    \quad \quad
    \zeta^{\mu_\alpha'} = \left\{\begin{array}{cc}
         1.1 \zeta^{\mu_\alpha'} & \text{if } V=1 \\
         \zeta^{\mu_\alpha'} & \text{otherwise}
    \end{array}\right.
\end{equation}
}

Time-varying confounding is introduced by modeling treatment assignment as a Bernoulli variable:

{
\small
\begin{equation}
A^{\mathrm{chemo}}_{t+1},A^{\mathrm{radio}}_{t+1} \sim \mathrm{Bernoulli} \left(\sigma \left(\frac{\gamma}{\mathrm{diam}(Y_{\mathrm{max}})} \left(\overline{\mathrm{diam}(Y_{>t-15})}- \frac{\mathrm{diam}(Y_{\mathrm{max}})}{2} \right)\right) \right)
\end{equation}
}

where $\sigma$ is the sigmoid function, $\mathrm{diam}(Y_{\mathrm{max}})$ is the maximum tumor diameter (13 cm), $\overline{\mathrm{diam}(Y_{>t-15})}$ is the average tumor diameter over the last 15 days. The value of $\gamma$ directly affects the strength of time-dependent confounding, i.e., the larger the value of $\gamma$, the larger the time-dependent confounding. 

Each tumor has an initial stage proportional to the value in \citet{detterbeckTurningGrayNatural2008}. Given the initial stage, the initial volume $Y_1$ is correspondingly sampled from the stage distribution as in \citet{gengPredictionTreatmentResponse2017}. For treatment initial value, $C_1=\tilde{C}_{1}=0$  mg/m$^3$ and $D_1=0$ Gy fractions.

Trajectories are measured until one of the following conditions happens: 1) The tumor diameter reaches $\mathrm{diam}(Y_{\mathrm{max}})=13$ cm. 2) Patient recovers with probability $\exp(-Y_t \zeta^\eta)$ where $\zeta^\eta$ is the tumor cell density as specified in \citet{gengPredictionTreatmentResponse2017}. 3) Termination after reaching the maximum time step $T=60$ days.

\subsection{Experimental details} \label{sec:nsclc-exp-details}

We generate 10000, 1000, and 1000 patients for training, validation, and testing. For training and testing, we simulate the full factual trajectories. All methods are tested on two settings: random trajectories and no confounding. We set $\tau=5$, i.e., 5-step ahead prediction.

\textbf{Random trajectories}. At every time step $t$, we evaluate the methods by predicting potential horizons given the history $\mathcal{H}_t$. Therefore, for each $t$, we need to simulate an exponential of $K^{\tau}$ outcomes for a prediction horizon of $\tau$, which is prohibitively expensive. We can randomly sample a fixed subset of $k$ outcomes for evaluation, which includes both factual and counterfactual outcomes. This setting is similar to the setting in \citet{melnychukCausalTransformerEstimating2022}. We set $k=1$ and simulate 1000 patients, resulting in 1000 test horizons per time step $t$. Since $T=60$ and $\tau=5$, we obtain $55$ steps for history, resulting in $55\times 1000=55 000$ test horizons. 

\textbf{No confounding}.  At every time step $t$, we evaluate the methods by predicting factual horizons given the history $\mathcal{H}_t$. As the parameter $\gamma$ can be controlled, we evaluate the model on a non-confounded ($\gamma=0$) test set. This setting is more difficult than random trajectories. In the random trajectories, the input history resembles the confounded training set distribution, while the test horizons do not. In the no-confounding setting, neither the history nor the test horizon resembles the training set distribution. For this strategy, we would only need to generate the factual outcomes for testing. The setting is similar to \citet{limForecastingTreatmentResponses2018}. We use the trajectories of 1000 patients. 

\section{MIMIC-III Dataset} \label{sec:mimic-synthetic}

\begin{wraptable}{r}{0.275\linewidth}
\tiny
\vspace{-15pt}
\centering
\caption{Selected time-varying covariates and static covariates from the MIMIC-Extract processed data.}
\begin{tabular}{@{}l@{}}
\toprule
\textbf{Time-varying covariates} \\ \midrule
\texttt{heart rate} \\
\texttt{red blood cell count} \\
\texttt{sodium} \\
\texttt{mean blood pressure} \\
\texttt{systemic vascular resistance} \\
\texttt{glucose} \\
\texttt{chloride urine} \\
\texttt{glascow coma scale total} \\
\texttt{hematocrit} \\
\texttt{positive end-expiratory pressure set} \\
\texttt{respiratory rate} \\
\texttt{prothrombin time pt} \\
\texttt{cholesterol} \\
\texttt{hemoglobin} \\
\texttt{creatinine} \\
\texttt{blood urea nitrogen} \\
\texttt{bicarbonate} \\
\texttt{calcium ionized} \\
\texttt{partial pressure of carbon dioxide} \\
\texttt{magnesium} \\
\texttt{anion gap} \\
\texttt{phosphorous} \\
\texttt{venous pvo2} \\
\texttt{platelets} \\
\texttt{calcium urine} \\ 
\midrule
\midrule
\textbf{Static covariates} \\
\midrule
\texttt{gender} \\
\texttt{ethinicity} \\
\texttt{age} \\ 
\bottomrule
\end{tabular}
\label{tab:mimic-covariates}
\vspace{-30pt}
\end{wraptable}

We generate semi-synthetic data based on the MIMIC-III dataset \citep{johnsonMIMICIIIFreelyAccessible2016} using a similar protocol to \citet{melnychukCausalTransformerEstimating2022} with some minor modifications to simplify the simulation. Similarly, the real-world data selection also follows \citet{melnychukCausalTransformerEstimating2022}. We used the MIMIC-Extract preprocessing pipeline \cite{wangMIMICExtractDataExtraction2020} to obtain the hourly aggregated data for the intensive care unit (ICU). All continuous time-varying covariates are further processed using forward filling, backward filling, and normalization.

From the hourly data, we extract 25 time-varying covariates $\bm{X}_t$ and 3 static covariates $\bm{V}$ as in \cref{tab:mimic-covariates}. The static covariates are one-hot encoded into a 44-dimensional feature vector. Patients with less than 20 hours of ICU stay are removed. 

\subsection{Semi-synthetic simulation details}

The ICU maximum stay is cut off at 100 hours, creating time-varying covariates $\bm{X}_t$ with a maximum time step of $T=100$. 

We generate $N^Y=2$ outcomes. For an outcome $m$, denoted as $Y_{i,t}^{(m)}$, the simulation first creates a non-treated outcomes $\tilde{Y}_{i,t}^{(m)}$. The non-treated outcome consists of an endogenous dependency by a combination of B-spline and Gaussian process. For every outcome simulated, we randomly sampled a subset of 10 time-varying covariates $\bm{X}^{(m)}_{i,t}$ to be the exogenous dependency. Assuming the outcome $m$ is affected by a set of treatments $\bm{A}^{(m)}_{i,t}$ with treatment effects $\bm{E}^{(m)}_{i,t}$ (defined later), we then have:

{
\small
\begin{gather}
    \tilde{Y}_{i,t}^{(m)} = \underbrace{u_i^{(m)}}_{\text{initial value}} + \underbrace{\alpha^{b} b_i(t) + \alpha^{g} g_i^{(m)}(t)}_{\text{endogeneous}} + \underbrace{\alpha^{f} f^{(m)}(\bm{X}^{(m)}_{i,t})}_{\text{exogeneous}} + \underbrace{\epsilon^{(m)}_{i,t}}_{\text{noise}} \\
    Y^{(m)}_{i,t} = \tilde{Y}^{(m)}_{i,t} + \bm{E}^{(m)\intercal}_{i,t} \bm{1}
\end{gather}
}

where $\alpha^b$, $\alpha^g$, and $\alpha^f$ are weight for each component. Each component is defined as follows. The initial value is defined as $u_i^{(m)}\sim \mathrm{Uniform}(-0.5, 0.5)$ The spline $b(\cdot) \sim \text{Uniform}\{b^{\text{stable}}(\cdot),b^{\text{fast decline}}(\cdot), b^{\text{slow decline}}(\cdot), b^{\text{fast increase}}(\cdot), b^{\text{slow increase}}(\cdot)\}$ is a cubic spline sampled uniformly from a set of five cubic splines representing different global trends. The Gaussian process instance $g_i^{m}(t) \sim \mathcal{GP}(0, \mathrm{Matern}_{\nu=2.5})$ represent smooth local trend. Exogenous dependency is introduced by a non-linear function of the selected covariates, constructed by random Fourier features approximating an RBF kernel $f^{(m)}(\bm{X}_{i,t}^{(m)})=\phi(\bm{X}_{i,t}^{(m)})^\intercal  \bm{w}$ as in \cite{hensmanVariationalFourierFeatures2018} with $\bm{w} \sim \text{Normal}(0, \sigma_w^2 \bm{I})$. A noise term is added to account for randomness $\epsilon^{(m)}_{i,t} \sim \mathrm{Normal}(0, 0.05^2)$.

We generate $N^A=3$ treatments. The confounded treatment $l$, denoted $A^{(l)}_{i,t}$ and its treatment effect $E^{(l)}_{i,t}$, is generated as:

{
\small
\begin{gather}
    P^{(l)}_{i,t+1} = \sigma \left(\gamma^Y \overline{\bm{Y}^{(l)}_{i,>t-w^{(l)}}} + \gamma^X f^{(l)}(\bm{X}_{i,t}) \right)  \\
    A^{(l)}_{i,t} \sim \mathrm{Bernoulli}(P_{i,t}^{(l)}) \\
    E^{(l)}_{i,t} = \sum_{v=0}^{\omega^{(l)}-1} \frac{P^{(l)}_{i,t-v} A^{(l)}_{i,t-v}b}{2^{-v/2}}
\end{gather}
}

where $\gamma^Y$ and $\gamma^X$ are confounding parameters. $\overline{\bm{Y}^{(l)}_{i,>t-w^{(l)}}}$ is the average outcomes over the last $w^{(l)}$ hours. Note that $\bm{Y}^{(l)}_{i,t}$ is a subset of treated outcomes that confound with the treatment. Similar to the synthetic outcome, we also sampled a subset of 10 time-varying covariates $\bm{X}^{(l)}_{i,t}$ to be the confounding covariates for each treatment. $f^{(l)}$ is a non-linear function similar to $f^{(m)}$. The treatment assignment $A^{(l)}$ is modeled as a Bernoulli random variable with probability $P^{(l)}_{i,t}$. The treatment effect $E^{(l)}_{i,t}$ s just the sum of all treatments used within $\omega^{(l)}$ hours with scaling factor $b$. All treatment has maximum effect right after use but decay exponentially with a half-life of 2 hours.

We can observe that each outcome $Y^{(m)}_{i,t}$ is affected by multiple treatment $\bm{A}^{(m)}_{i,t}$ and each treatment $A^{(l)}_{i,t}$ is confounded with multiple outcomes $\bm{Y}^{(l)}_{i,t}$, creating a complex dynamic.

\subsection{Semi-synthetic experiment details}

We simulate 2000 and 200 patients for training and validation. For training and testing, we simulate the full factual trajectories. For testing, we use the random trajectories strategy similar to \cref{sec:nsclc-exp-details}. We set $\tau=10$, i.e., 10-step ahead prediction. Note that while the semi-synthetic MIMIC-III simulation also has a confounding parameter $\gamma^Y$ and $\gamma^X$, the treatment mechanism is different for different values of $\gamma^Y$ and $\gamma^X$. Therefore, we do not apply the no-confounding setting test to this simulation.

\subsection{Real-world data details}

For real-world data, we extract two binary treatments, including \texttt{\small vaso} (vasopressor $A_t^{\mathrm{vaso}}$) and \texttt{\small vent} (mechanical ventilation $A_t^{\mathrm{vent}}$), and two outcomes $\bm{Y}_t$, including \texttt{\small diastolic blood pressure} and \texttt{\small oxygen saturation}. The ICU maximum stay is cut off at 60 hours, creating trajectories with a maximum time step of $T=60$. 

\subsection{Real-world experiment details}

We select a subset of 5000 and 500 patients for training and validation. At every time step, we evaluate the methods by predicting the factual horizon $\tau$ given the history $\mathcal{H}_t$. We select a different subset of 500 patients for testing. We set $\tau=5$, i.e., 5-step ahead prediction.

\section{Baseline Methods}

\subsection{Common structure}

To ensure a fair comparison for all methods, we use a common structure including several components similar to \cref{fig:architecture}A. All methods include a representation function $\Phi$ and an outcome prediction head $F^Y$ (\cref{tab:arch-fy}). For adversarial-based methods like {\method}, CRN, and CT, an additional balancing head $F^B$ is added (\cref{tab:arch-fb}). $F^Y$ and $F^B$ share the same architectures for all methods, with the only potential difference being the input size for $F^Y$ and $F^B$. All methods are trained with Adam optimizers \cite{kingmaAdamMethodStochastic2017} and the teacher forcing technique \cite{williamsLearningAlgorithmContinually1989}. The dimension of $\Phi(\mathcal{H}_t)$ is referred to as \# hidden units in the hyperparameter search table in \cref{sec:hyperparams}. \# hidden units are selected so that all methods have roughly the same total parameter range.

\begin{table}[!ht]
\small
\centering
\caption{}
\begin{subtable}{0.3\textwidth} 
\centering
\caption{Architecture for $F^Y$} 
\begin{tabular}{@{}c@{}} 
\toprule
\makecell{\textbf{Input}: $[\Phi(\mathcal{H})_t, A_{t+1}]$ \\ or $\Phi(\mathcal{H})_t$ or $F^C(\Phi(\mathcal{H})_t, A_{t+1})$} \\ \midrule
Linear Layer \\
ELU Activation \\
Linear Layer \\
\midrule
\textbf{Output}: $\hat{\bm{Y}}_{t+1}$ or $\hat{\bm{Y}}_t$\\
\bottomrule
\end{tabular} 
\label{tab:arch-fy}
\end{subtable}
\quad
\begin{subtable}{0.2\textwidth}
\centering
\caption{Architecture for $F^B$}
\begin{tabular}{@{}c@{}}
\toprule
\textbf{Input}: $\Phi(\mathcal{H})_t$  \\ \midrule
Linear \\
ELU \\
Linear \\
Softmax \\
\midrule
\textbf{Output}: $\hat{A}_{t+1}$ \\
\bottomrule
\end{tabular}
\label{tab:arch-fb}
\end{subtable}
\end{table}

\subsection{\method}

\textbf{Architecture}. {\method} uses a decoder-only architecture. The representation network $\Phi$ for {\method} is parametrized by either LSTM cells \cite{hochreiterLongShortTermMemory1997} layers or TCN residual blocks \cite{baiEmpiricalEvaluationGeneric2018}. We refer to both as layers. Since TCN is composed of dilated causal convolutions with an exponential receptive field as a function of depth, TCN requires more layers to cover the entire history. The architecture for $\Phi$ is in \cref{tab:arch-caetc}. The output dimension for both projection layers is set to \# hidden units.

\begin{table}[!ht]
\small
\centering
\caption{}
\begin{subtable}{0.3\textwidth}
\centering
\caption{Architecture for $\omega^{\bm{X}, A, \bm{Y}}$}
\begin{tabular}{@{}c@{}}
\toprule
\textbf{Input}: $[\bm{V}, A_t, \bm{Y}_t]$  \\ \midrule
Linear Layer \\
\midrule
\textbf{Output}: $\omega^{\bm{X}, A, \bm{Y}}([\bm{V}, A_t, \bm{Y}_t])$ \\
\bottomrule
\end{tabular}
\end{subtable}
\quad
\begin{subtable}{0.2\textwidth}
\centering
\caption{Architecture for $\omega^{\bm{X}}$}
\begin{tabular}{@{}c@{}}
\toprule
\textbf{Input}: $\bm{X}_t$  \\ \midrule
Linear Layer \\
\midrule
\textbf{Output}: $\omega^{\bm{X}}(\bm{X}_t)$ \\
\bottomrule
\end{tabular}
\end{subtable}
\quad
\begin{subtable}{0.3\textwidth}
\centering
\caption{Architecture for $\Phi$ ({\method})}
\begin{tabular}{@{}c@{}}
\toprule
\textbf{Input}: $\mathcal{H}_t$ \\ \midrule
$\omega^{\bm{X}, A, \bm{Y}}$ and $\omega^{\bm{X}}$ \\
LSTM (\# layers) or TCN (\# layers)\\
Linear Layer \\
ELU Activation\\
\midrule
\textbf{Output}: $\Phi(\mathcal{H})_t$\\
\bottomrule
\end{tabular}
\label{tab:arch-caetc}
\end{subtable}
\end{table}

For {\method}, two additional heads $F^A$ (\cref{tab:arch-fa}) and $F^X$ (\cref{tab:arch-fx}) and a conditioning layer $F^C$ (\cref{tab:arch-fc}) are added.

\begin{table}[!ht]
\small
\centering
\caption{}
\begin{subtable}{0.3\textwidth}
\centering
\caption{Architecture for $F^A$}
\begin{tabular}{@{}c@{}}
\toprule
\textbf{Input}: or $\Phi(\mathcal{H})_t$ or $F^C(\Phi(\mathcal{H})_t, A_{t+1})$ \\ \midrule
Linear Layer \\
ELU Activation\\
Linear Layer \\
Softmax \\
\midrule
\textbf{Output}: $\hat{A}_t$ or $\hat{A}_{t+1}$\\
\bottomrule
\end{tabular}
\label{tab:arch-fa}
\end{subtable}
\quad
\begin{subtable}{0.2\textwidth}
\centering
\caption{Architecture for $F^X$}
\begin{tabular}{@{}c@{}}
\toprule
\textbf{Input}: $\Phi(\mathcal{H})_t$  \\ \midrule
Linear Layer \\
ELU Activation \\
Linear Layer \\
\midrule
\textbf{Output}: $\hat{\bm{X}}_t$ \\
\bottomrule
\end{tabular}
\label{tab:arch-fx}
\end{subtable}
\quad
\begin{subtable}{0.25\textwidth}
\centering
\caption{Architecture for $F^C$}
\begin{tabular}{@{}c@{}}
\toprule
\textbf{Input}: $\Phi(\mathcal{H}_t)$ and $a^{(k)}$  \\ \midrule
2x Embedding for $R^\xi$ and $R^\beta$\\
\midrule
\textbf{Output}: $\Phi(\mathcal{H}_t) \odot \xi^{(k)} \oplus \beta^{(k)}$ \\
\bottomrule
\end{tabular}
\label{tab:arch-fc}
\end{subtable}
\end{table}

\textbf{Training}. The training process is in \cref{alg:training}. We set $\delta^A=0.1$, $\delta^X=0.1$ and $\delta^B=0.0001$ for all experiments.

\begin{algorithm}[!ht] 
  \caption{Training algorithm for {\method}}
  \begin{algorithmic} \label{alg:training}
    \STATE \textbf{Input}: Dataset $\dataset{D}=\left\{\left\{A_{i,t}, \bm{Y}_{i,t}, \bm{X}_{i,t}\right\}_{t=1}^{T}, \bm{V}_{i}\right\}_{i=1}^N$, hyperparameters $\delta^A$, $\delta^X$, $\delta^E$
    \FOR{epoch $e_j=1$ to $e_{\mathrm{max}}$}
        \STATE Compute loss $\mathcal{L}(\theta_{F^A}, \theta_{F^Y}, \theta_{F_X}, \theta_{F^C}, \theta_{\Phi})$
        {
        \vspace{-8pt}
        \begin{equation*}
            = \big[\mathcal{L}^{RY}(\theta_{F^Y}, \theta_\Phi) + \mathcal{L}^{Y}(\theta_{F^Y},\theta_{F^C},\theta_\Phi)\big] + \delta^A \big[\mathcal{L}^{RA}(\theta_{F^A},\theta_\Phi)+\mathcal{L}^{C}(\theta_{F^A},\theta_{F^C},\theta_\Phi)\big] 
            + \delta^X \mathcal{L}^{RX}(\theta_{F^X}, \theta_\Phi) + \delta^E \mathcal{L}^E(\theta_\Phi)
        \end{equation*}
        \vspace{-20pt}
        }
        \STATE Compute gradient $\nabla_{\theta_{F^A}, \theta_{F^Y}, \theta_{F_X}, \theta_{F^C}, \theta_{\Phi}} [\mathcal{L}(\theta_{F^A}, \theta_{F^Y}, \theta_{F^X}, \theta_{F^C}, \theta_{\Phi})]$
        \STATE Compute loss $\delta^E \mathcal{L}^B(\theta_{F^B})$
        \STATE Compute gradient $\nabla_{\theta_{F^B}} [\delta^E \mathcal{L}^B(\theta_{F^B})]$
        \STATE Update parameters $\theta_{F^A}, \theta_{F^Y}, \theta_{F^X}, \theta_{F^C}, \theta_{\Phi}$ by gradient $\nabla_{\theta_{F^A}, \theta_{F^Y}, \theta_{F_X}, \theta_{F^C}, \theta_{\Phi}} [\mathcal{L}(\theta_{F^A}, \theta_{F^Y}, \theta_{F^X}, \theta_{F^C}, \theta_{\Phi}))]$
        \STATE Update parameters $\theta_{F^B}$ by gradient $\nabla_{\theta_{F^B}} [\delta^E \mathcal{L}^B(\theta_{F^B})]$
    \ENDFOR
  \end{algorithmic}
  
\end{algorithm}

\subsection{RMSN}

\begin{wraptable}{r}{0.3\linewidth}
\small
\centering
\caption{Architecture for $\Phi$ (RMSN)}
\begin{tabular}{@{}c@{}}
\toprule
\textbf{Input}: $\{A_{t'}\}_{t'=1}^t$ or  $\mathcal{H}_t$  \\ \midrule
LSTM (\# layers)\\
Linear Layer \\
ELU Activation\\
\midrule
\textbf{Output}: $\Phi(\{A_{t'}\}_{t'=1}^t)$ or $\Phi(\mathcal{H}_t)$\\
\bottomrule
\end{tabular}
\label{tab:arch-rmsn}
\end{wraptable}

RMSN requires the estimation of $P(A_{t+1}|A_t)$ and $P(A_{t+1}|\mathcal{H}_t)$ for IPTW, each of which is parametrized by an LSTM. They are referred to as the propensity-treatment network and the propensity-history network. RMSN uses an encoder-decoder architecture. The propensity-treatment network and propensity-history network are both composed of a representation function $\Phi$ and a treatment classification head, for which we reuse the architecture of $F^B$ (\cref{tab:arch-fb}). The encoder and decoder are both composed of a representation function $\Phi$ and an outcome prediction head $F^Y$ (\cref{tab:arch-fy}). The architecture of $\Phi$ is similar for the propensity-treatment network, the propensity-history network, the encoder, and the decoder as in \cref{tab:arch-rmsn}. In total, RMSN requires four different individual networks for training.

\subsection{CRN}

\begin{wraptable}{r}{0.3\linewidth}
\small
\vspace{-30pt}
\centering
\caption{Architecture for $\Phi$ (CRN)}
\begin{tabular}{@{}c@{}}
\toprule
\textbf{Input}: $\mathcal{H}_t$  \\ \midrule
LSTM (\# layers)\\
Linear Layer \\
ELU Activation\\
\midrule
\textbf{Output}: $\Phi(\mathcal{H})_t$\\
\bottomrule
\end{tabular}
\label{tab:arch-crn}
\vspace{-20pt}
\end{wraptable}

CRN uses an encoder-decoder architecture with a gradient reversal layer for adversarial training. Both the encoder and decoder are composed of a representation function $\Phi$ (\cref{tab:arch-crn}), an outcome prediction head $F^Y$ (\cref{tab:arch-fy}), and a treatment balancing head $F^B$ (\cref{tab:arch-fb}). CRN requires a hyperparameter for adversarial training, which is similar to our $\delta^B$. We set CRN's $\delta^B=0.1$.

\subsection{CT}

\begin{wraptable}{r}{0.3\linewidth}
\small
\vspace{-10pt}
\centering
\caption{Architecture for $\Phi$ (CT)}
\begin{tabular}{@{}c@{}}
\toprule
\textbf{Input}: $\mathcal{H}_t$  \\ \midrule
3x Transformers (\# layers)\\
\midrule
\textbf{Output}: $\Phi(\mathcal{H})_t$\\
\bottomrule
\end{tabular}
\label{tab:arch-ct}

\end{wraptable}

CT uses a decoder-only architecture and domain confusion loss for adversarial training. The decoder is composed of a representation function $\Phi$ (\cref{tab:arch-ct}), an outcome prediction head $F^Y$ (\cref{tab:arch-fy}), and a treatment balancing head $F^B$ (\cref{tab:arch-fb}). CT also requires an exponential moving average (EMA) copy of the decoder. We reuse original hyperparameters in CT for the EMA copy. CT requires a hyperparameter for adversarial training, which is similar to our $\delta^B$. We set CT's $\delta^B=0.01$.

\subsection{Adversarial weight $\delta^B$}

For all methods that use $\delta^B$, we perform an exponential increase in the value of $\delta^B_e$ as a function of the epoch until we reach $\delta^B$.

\begin{equation}
    \delta^B_{e_j} = \delta^B \left(\frac{2}{1+\exp(-10\times e_j/e_{\mathrm{max}})} - 1\right)
\end{equation}

\section{Hyperparameter Search} \label{sec:hyperparams}

\begin{table}[!ht]
\small
\centering
\begin{threeparttable}[c]
\caption{Hyperparameter search range for every method \& dataset. If multiple components (e.g., encoder \& decoder) are in the same cell, the hyperparameter is shared for those components (e.g., encoder and decoder have the same number of layers) for the same run. Otherwise, the hyperparameter is component-specific for the same run. We run a random grid search with 10 combinations for every method for fair comparison. The final results reported are averaged over 3 random seeds.} 
\begin{tabular}{@{}c|c|c|cc@{}}
\toprule
Method & Components & Hyperparameters & NSCLC range & \makecell{MIMIC III range  \\ (semi-synthetic \& real)} \\ 
\midrule
\midrule
\multirow{9}{*}{\makecell{LSTM \\ \& CRN}} & \multirow{3}{*}{\makecell{Encoder \\ \& Decoder}} & \# layers & 1 & 1,2 \\
 &  & \# hidden units & 16, 32, 48 & 32, 48, 64 \\
 &  & Dropout rate & \multicolumn{2}{c}{0.0, 0.1}  \\
\cmidrule{2-5}
 & \multirow{3}{*}{Encoder} & \# epochs & 100 & 150 \\
 &  & Learning rate & \multicolumn{2}{c}{0.001, 0.0001} \\
 &  & Batch size & \multicolumn{2}{c}{64, 128} \\
\cmidrule{2-5}
 & \multirow{3}{*}{Decoder} & \# epochs & 50 & 100 \\
 &  & Learning rate & \multicolumn{2}{c}{0.001, 0.0001}  \\
 &  & Batch size & \multicolumn{2}{c}{512, 1024} \\
\midrule
\multirow{7}{*}{RMSN} & \multirow{3}{*}{\makecell{Propensity \tnote{1} \\ \& Encoder \\ \& Decoder}} & \# layers & 1 & 1,2 \\
 &  & \# hidden units & 16, 32, 48 & 32, 48, 64 \\
 &  & Dropout rate & \multicolumn{2}{c}{0.0, 0.1}  \\
\cmidrule{2-5}
 & \multirow{2}{*}{\makecell{Propensity\tnote{1} \\ \& Encoder}} & \# epochs & 100 & 150 \\
 &  & Learning rate & \multicolumn{2}{c}{0.001, 0.0001} \\
\cmidrule{2-5}
 & \multirow{2}{*}{Decoder} & \# epochs & 50 & 100 \\
 &  & Learning rate & \multicolumn{2}{c}{0.001, 0.0001}  \\
\midrule
\multirow{8}{*}{CT} & & \# layers & 1 & 1,2 \\
 &  & \# hidden units & 16, 24, 32 & 18, 24, 36 \\
 &  & \# attention heads & 2  & 2,3 \\
 &  & Max positional encoding & 15 & 30 \\
 &  & \# epochs & 150 & 250 \\
 &  & Dropout rate & \multicolumn{2}{c}{0.0, 0.1}  \\
 &  & Learning rate & \multicolumn{2}{c}{0.001, 0.0001} \\
 &  & Batch size & \multicolumn{2}{c}{64, 128} \\
\midrule
\multirow{6}{*}{\makecell{\method{}\\ \tiny LSTM}} & & \# layers & 1 & 1,2 \\
 &  & \# hidden units & 16, 32, 48 & 32, 48, 64 \\
 &  & \# epochs & 150 & 250 \\
 &  & Dropout rate & \multicolumn{2}{c}{0.0, 0.1}  \\
 &  & Learning rate & \multicolumn{2}{c}{0.001, 0.0001} \\
 &  & Batch size & \multicolumn{2}{c}{64, 128} \\
\midrule
\multirow{8}{*}{\makecell{\method{}\\ \tiny TCN}} & & \# layers & 3,4, 5 & 4,5, 6 \\
 &  & \# hidden units & 16, 24, 32 & 24, 32, 40 \\
 &  & Kernel size & \multicolumn{2}{c}{3, 5} \\
 &  & Dilation factor & \multicolumn{2}{c}{2} \\
 &  & \# epochs & 150 & 250 \\
 &  & Dropout rate & \multicolumn{2}{c}{0.0, 0.1}  \\
 &  & Learning rate & \multicolumn{2}{c}{0.001, 0.0001} \\
 &  & Batch size & \multicolumn{2}{c}{64, 128} \\
\midrule
 \bottomrule
\end{tabular}
\begin{tablenotes}
   \item [1] Both the propensity-treatment \& propensity-history networks are referred to as ``propensity''.
 \end{tablenotes}
\end{threeparttable}
\end{table}

\end{document}